\documentclass[twoside]{article}
\usepackage[accepted]{aistats2023}
\usepackage{other_files/alias}

\usepackage[utf8]{inputenc}
\usepackage{amsmath,amssymb,amsfonts,amsthm}
\usepackage{thm-restate}
\usepackage{mathtools}
\usepackage{mdframed}
\usepackage{algorithm}
\usepackage{algpseudocode}
\usepackage[colorlinks,citecolor=blue,urlcolor=blue,bookmarks=false,hypertexnames=true,linkcolor=blue]{hyperref}

\theoremstyle{plain}
\newtheorem{theorem}{Theorem}[section]

\newtheorem{lemma}[theorem]{Lemma}

\theoremstyle{definition}

\theoremstyle{remark}

\declaretheorem[name=Proposition,numberwithin=section,style=examplestyle]{prop}

\usepackage[round]{natbib}

\begin{document}



\twocolumn[

\aistatstitle{Robust Linear Regression: Gradient-descent, Early-stopping, and Beyond}

\aistatsauthor{Meyer Scetbon \And Elvis Dohmatob}

\aistatsaddress{Facebook AI Research}]

% \addtocontents{toc}{\protect\setcounter{tocdepth}{0}}
\begin{abstract}
In this work we study the robustness to adversarial attacks, of early-stopping strategies on gradient-descent (GD) methods for linear regression. More precisely, we show that early-stopped GD is optimally robust (up to an absolute constant) against Euclidean-norm adversarial attacks. However, we show that this strategy can be arbitrarily sub-optimal in the case of general Mahalanobis attacks. This observation is compatible with recent findings in the case of classification~\cite{Vardi2022GradientMP} that show that GD provably converges to non-robust models. To alleviate this issue, we propose to apply instead a GD scheme on a transformation of the data adapted to the attack. This data transformation amounts to apply feature-depending learning rates and we show that this modified GD is able to handle any Mahalanobis attack, as well as more general attacks under some conditions. Unfortunately, choosing such adapted transformations can be hard for general attacks. To the rescue, we design a simple and tractable estimator whose adversarial risk is optimal up to within a multiplicative constant of 1.1124 in the population regime, and works for any norm.
\end{abstract}

\setcounter{tocdepth}{0} %% uncomment line for non-verbose ToC

\section{Introduction}
Machine learning models are highly sensitive to small perturbations known as \emph{adversarial examples}~\citep{szegedy2013intriguing}, which are often imperceptible by humans.
While various strategies such as adversarial training~\citep{madry2017towards} can mitigate this vulnerability empirically, % lack of robustness
the situation remains highly problematic for many safety-critical applications like autonomous vehicles or health, and motivates a better theoretical understanding of what mechanisms may be causing this.

From a theoretical perspective, the case of classification is rather well-understood. Indeed, the hardness of classification under test-time adversarial attacks has been crisply characterized \citep{bhagoji2019,bubeck2018}. In the special case of linear classification, explicit lower-bounds have been obtained \citep{schmidt2018,bhattacharjee2020sample}. However, the case of regression is relatively understudied. Recently, \cite{Xing2021} have initiated a theoretical study of linear regression under Euclidean attacks, where an adversary is allowed to attack the input data point at test time. The authors proposed a two-stage estimator and proved its consistency. The optimal estimator obtained in \citep{Xing2021} corresponds to a ridge shrinkage (i.e $\ell_2$ penalization).

In this paper, we consider linear regression under adversarial test-time attacks w.r.t arbitrary norms (not just Euclidean / $\ell_2$-norms as in \citep{Xing2021}), and analyze the robustness of gradient-descent (GD) along the entire optimization path. By doing so we observe that GD might fail to capture a robust predictor along its path especially in the case of non-Euclidean attacks. We propose a variant of the GD scheme where an adapted transformation on the data is performed before applying a GD scheme. This allows us to understand the effect on robustness of early-stopping strategies (not training till the end) for general attacks. Finally we design a generic algorithm able to produce a robust predictor against any norm-attacks.

\subsection{Summary of main contributions}
Our main contributions are summarized as follows.

-- \textbf{Case of gradient-descent (GD).}
In Proposition~\ref{prop:GDisOpt}, we show that early-stopped GD achieves near-optimal adversarial risk in case of Euclidean attacks. Early-stopping is crucial because the predictor obtained by running GD till the end can be arbitrarily sub-optimal in terms of adversarial risk (Proposition~\ref{prop:generative}). Contrasting with Proposition~\ref{prop:GDisOpt}, we show in Proposition~\ref{thm:GDisBadforMan} that early-stopped GD can be arbitrarily sub-optimal in the non-Euclidean case, e.g when the attacker's norm is a Mahalanobis norm. Thus, GD, along its entire optimization path, can fail to find robust model in general.

-- \textbf{An Adapted GD scheme (GD+).} We propose a modified version of GD, termed GD+, in which the dynamics
are forced to be non-uniform across different features and
exhibit different regimes where early-stopped GD+ can achieves near-optimal adversarial risk. More precisely, (i) we show that it achieves near-optimal adversarial risk in the case of Mahalanobis norm attacks (Proposition~\ref{prop:GDPlusIsGood}), (ii) we also prove that it is near-optimally robust under $\ell_p$-norm attacks as soon as the features are uncorrelated (Theorem~\ref{thm-lp-optimal}) and (iii) we study the robustness along the entire optimization path of GD+ in the case of general norm-attacks. In particular, we provide a sufficient condition on the model such that early-stopped GD+ achieves near-optimal adversarial risk under general norm-attacks in Theorem~\ref{thm:GDrocksdeterministic} and we show that when this condition is not satisfied, GD+ can be arbitrarily sub-optimal (Proposition~\ref{thm:bangbang}).

% -- \textbf{An Adapted GD scheme (GD+).} We propose a modified version of GD, termed GD+, in which the dynamics are forced to be non-uniform across different features, and show that it achieves optimal adversarial risk in the case of Mahalanobis norm attacks (Proposition~\ref{prop:GDPlusIsGood}).

% -- \textbf{Sufficient conditions for optimality of GD+.}
% Then we study the effect of whitening on robustness under general attacks as a special case of GD+. We  provide a sufficient condition on the generative model $w_0$ such that this specific GD+ achieves optimal adversarial risk in Theorem \ref{thm:GDrocksdeterministic}. We also show that when this condition fails, the whitening strategy can be arbitrarily sub-optimal (Proposition~\ref{thm:bangbang}).

% -- \textbf{Optimality under $\ell_{p}$-norm-attacks.}
% When we restrict the attacks to be $\ell_p$-attacks with $p\in[1,+\infty]$, we show in Theorem~\ref{thm-lp-optimal} that early stopped GD+ is able to reach near-optimal robustness as soon as the features are uncorrelated.

-- \textbf{A two-stage estimator for the general case.}
Finally, we propose a simple two-stage algorithm (Algorithm \ref{alg:twostage}) which works for arbitrary norm-attacks, and achieves optimal adversarial risk up to within a multiplicative constant factor in the population regime. Consistency and statistical guarantees for the proposed estimator are also provided.
\subsection{Related Work}
The theoretical understanding of adversarial examples is now an active area of research. Below is a list of works which are most relevant to our current letter.

%-- \emph{Robustness and Accuracy Could Be Reconcilable by (Proper) Definition} \citep{Pang2022RobustnessAA}. Here, adversarial risk is defined it has the same minimizer as standard risk, namely the generative model $f_\star(x) := \mathbb E[y \mid x]$.

\cite{tsipras18} considers a specific data distribution where good accuracy implies poor robustness.
% For example, 
\citep{goldstein,saeed2018, gilmerspheres18,dohmatob19} show that for high-dimensional data distributions which have concentration property (e.g., multivariate Gaussians, distributions satisfying log-Sobolev inequalities, etc.), an imperfect classifier will admit adversarial examples. \cite{dobriban2020provable} studies tradeoffs in Gaussian mixture classification problems, highlighting the impact of class imbalance. On the other hand, \cite{closerlook2020} observed empirically that natural images are well-separated, and so locally-lipschitz classifies shouldn't suffer any kind of test error vs robustness tradeoff.

In the context of linear classification\citep{schmidt2018,bubeck2018,khim2018adversarial,yin2019rademacher,bhattacharjee2020sample,min2021curious,Karbasi2021}, established results show a clear gap between learning in ordinary and adversarial settings.
\cite{li2020} studies the dynamics of linear classification on separable data, with exponential-tail losses. The authors show that GD converges to separator of the dataset, which is minimal w.r.t to a norm which is an interpolation between the the $\ell_2$ norm (reminiscent of normal learning), and $\ell_q$-norm, where $q$ is the harmonic conjugate of the attacker's norm. \cite{Vardi2022GradientMP} showed that on two-layer neural networks, gradient-descent with exponential-tail loss function converges to weights which are vulnerable to adversarial examples.

\cite{Javanmard2020PreciseTI} study tradeoffs between ordinary and adversarial risk in linear regression, and computed exact Pareto optimal curves. \cite{javanmard2021adversarial} also revisit this tradeoff for latent models and show that this tradeoff is mitigated when the data enjoys a low-dimensional structure. \cite{dohmatob2021fundamental,hassani2022curse} study the tradeoffs between interpolation, normal risk, and adversarial risk, for finite-width over-parameterized networks with linear target functions. \cite{javanmard2022precise} investigate the effect of adversarial training on the standard and adversarial risks and derive a precise characterization of them for a class of minimax adversarially trained models.  

The work most related to ours is \citep{Xing2021} which studied minimax estimation of linear models under adversarial attacks in Euclidean norm. They showed that the optimal robust linear model is a ridge estimator whose regularization parameter is a function of the population covariance matrix $\Sigma$ and the generative linear model $w_0$. Since neither $w_0$ nor $\Sigma$ is known in practice, the authors proposed a two-stage estimator in which the first stage is consistent estimators for $w_0$ and $\Sigma$ and the second stage is solving the ridge problem. In \citep{Xing2021}, the authors cover only the case of Euclidean attacks while here, we extend the study of the adversarial risk in linear regression under general norm-attacks. More precisely, here we are interested in understanding the robustness of GD, for general attacks and along the entire optimization path. As a separate contribution, we also propose a new consistent two-stage estimator based on a "dualization" of the adversarial problem that can be applied for general attacks.

% Their work has several notable limitations
% \begin{itemize}
% \item The computation of their proposed estimator is not tractable. It relies on solving a complicated equation to find an optimal ridge penalty parameter used in constructed the estimator. The only exception is when the population covariance matrix $\Sigma$ is known to be the identity matrix.
% \item In terms of scope, the \citep{Xing2021} covers the case of Euclidean attacks.
% \end{itemize}
\subsection{Outline of Manuscript}
In Section \ref{sec:preliminaries}, we present the problem setup, main definitions, and some preliminary computations. The adversarial risk of (early-stopped) GD in the infinite-sample regime is analyzed in Section \ref{sec:gd}; cases of optimality and sub-optimality of this scheme are characterized. In Section \ref{sec:gdplus}, GD+ (an improved version of GD) is proposed, and its adversarial risk in the population regime is studied. In Section \ref{sec:algorithmic}, we consider the finite samples regime and we propose a simple two-stage estimator, which works for all attacker norms. Its adversarial risk is shown to be optimal up to within a multiplicative factor, and additive statistical estimation error due to finite samples.
\section{Preliminaries}
\label{sec:preliminaries}
\paragraph{Notations.}
Let us introduce some basic notations. Additional technical notations are provided in the appendix.

$[d]$ denotes the set of integers from $1$ to $d$ inclusive. The maximum (resp. minimum) of two real numbers $a$ and $b$ will be denoted $a \lor b$ (resp. $a \land b$). The operator norm of a matrix $A$ is denoted $\|A\|_{op}$ and corresponds to the positive square-root of the largest eigenvalue of $AA^\top$. Given a positive-definite (p.d.) matrix $S \in \mathbb R^{d \times d}$ and a vector $w \in \mathbb R^d$, the Mahalanobis norm of $w$ induced by $S$ is defined by $\|w\|_S := \sqrt{w^\top S w}$.  We denote respectively $M_d(\mathbb R)$ and $\mathcal S_d^{++}(\mathbb R)$ the set of $d\times d$ matrices and positive-definite matrices. Given $p \in [1,\infty]$, its harmonic conjugate, denoted $q(p)$ is the unique $q \in [1,\infty]$ such that $1/p+1/q=1$. 
% For example, $1$ and $\infty$ are harmonic conjugates of one another, while $2$ is its own harmonic conjugate. 
% The orthogonal projection operator onto a closed convex subset of $\mathbb R^d$ is denoted $\mbox{proj}_C$.
% The indicator function of a subset $C$ of $\mathbb R^d$ is the function $i_C:\mathbb R^d \to (-\infty,\infty]$, defined by $\sigma_C(x) = 0$ if $x \in C$, and $\sigma_C(x) = \infty$ otherwise.
% The dual $\|\cdot\|_\star$ of a norm $\|\cdot\|$ on $\mathbb R^d$, denoted $\|\cdot\|_\star$, is the norm defined by
% \begin{eqnarray}
% \|w\|_\star := \sup_{z \in % B^d_{\|\cdot\|}}z^\top w,
% \end{eqnarray}
% Thus, $\|\cdot\|_\star$ and $\|\cdot\|$ are support functions of one another.
For a given norm $\Vert\cdot\Vert$, we denote $\|\cdot\|_\star$ its dual norm, defined by
\begin{eqnarray}
 \|w\|_\star := \sup\{x^\top w \mid x  \in \mathbb R^d,\,\|x\| \le 1\},
\end{eqnarray}
Note that if the norm $\|\cdot\|$ is an $\ell_p$-norm for some $p \in [1,\infty]$, then the dual norm is the $\ell_q$-norm, where $q$ is the harmonic conjugate of $p$. For example, the dual of the euclidean norm (corresponding to $p=2$) is the euclidean norm itself, while the dual of the usual $\ell_\infty$-norm is the $\ell_1$-norm.

The unit-ball for a norm $\|\cdot\|$ is denoted $B_{\|\cdot\|}^d$, and defined by $B_{\|\cdot\|}^d := \{x \in \mathbb R^d \mid \|x\| \le 1\}$.
% When $\|\cdot\|$ is an $\ell_p$-norm, we write $B_p^d$ for $B_{\|\cdot\|_p}^d$.
Given any $s \ge 0$, set of $s$-sparse vectors in $\mathbb R^d$ is denoted $B_0^d$ and defined by
\begin{eqnarray}
\label{eq-sparse-ball}
B_0^d(s) := \{w \in \mathbb R^d \mid \|w\|_0 \le s\}.
\end{eqnarray}
% Note that $B_0^d(0) = 0$ and $B_0^d(s)=\mathbb R^d$ for all $s \ge d$.
% The usual notation $\mathcal O_d(1)$ (resp. $\mathcal O_{d,\mathbb P}(1)$) is used to denote a quantity which remains bounded (resp. bounded in probability) in the limit $d \to \infty$. Likewise $o_{d}(1)$ (resp. $o_{d,\mathbb P}(1)$) denotes a quantity which goes to zero (resp. which goes to zero in probability) in the limit $d \to \infty$. As usual, the acronym "a.s" stands for \emph{almost-surely}, while "w.p $p$" stands for \emph{with probability at least $p$}.

Finally, define absolute constants $c_0:=\sqrt{2/\pi}$, $\alpha := 2/(1+c_0) \approx 1.1124$ and $\beta=1.6862$.

\subsection{Problem Setup}
\label{subsec:setup}
Fix a vector $w_0 \in \mathbb R^d$, a positive definite matrix $\Sigma$ of size $d$, and consider an i.i.d. dataset $\mathcal D_n = \{(x_1,y_1),\ldots,(x_n,y_n)\}$ of size $n$, given by
\begin{align}
\label{eq:generative}
y_i &= x_i^\top w_0 + \epsilon_i,\text{ for all }i \in [n]
\end{align}
where $x_1,\ldots,x_n$ i.i.d $\sim N(0, \Sigma)$ and $\epsilon_1,\ldots,\epsilon_n$ i.i.d $\sim N(0,\sigma_\epsilon^2)$ independent of  the $x_i$'s.
% \begin{eqnarray}
% \begin{split}
% x_1,\ldots,x_n &\overset{iid}{\sim} N(0, \Sigma),\\
% \epsilon_1,\ldots,\epsilon_n &\overset{iid}{\sim} N(0,\sigma_\epsilon^2),\text{ independent of  the }x_i\text{'s},\\
% y_i &= x_i^\top w_0 + \epsilon_i,\text{ for all }i \in [n].
% \end{split}
% \label{eq:generative}
% \end{eqnarray}
Thus, the distribution of the features is a centered multivariate Gaussian distribution with covariance matrix $\Sigma$, while $w_0$ is the generative model. $\sigma_\epsilon \ge 0$ measures the size of the noise. 
% A priori, we will not assume that $w_0$ or $\Sigma$ are known.
These assumptions on the model will be assumed along all the formal claims of the paper. We also refer to $n$ as the  \emph{sample size}, and to $d$ as the \emph{input-dimension}.

In most of our analysis, except otherwise explicitly stated,  we will consider the case of infinite-data regime $n = \infty$ --or more generally $n \gg d$, which allows us to focus on the effects inherent to the data distribution (controlled by feature covariance matrix $\Sigma$) and the inductive bias of the norm w.r.t which the attack is measured, while side-stepping issues due to finite samples and label noise. Also
note that in this infinite-data setting, label noise provably has no influence on the learned model.

% For notational simplicity, write $X:=(x_1,\ldots,x_n) \in \mathbb R^{n \times d}$ (the design matrix), $y:=(y_1,\ldots,y_n) \in \mathbb R^n$ (the response vector), and $\epsilon := (\epsilon_1,\ldots,\epsilon_n) \in \mathbb R^n$ (the noise vector).

\subsection{Adversarial Robustness Risk}
Given a linear model $w \in \mathbb R^d$, an attacker is allowed to swap a clean test point $x \sim N(0,\Sigma)$ with a corrupted version $x' = x + \delta$ thereof. The perturbation $\delta=\delta(x) \in \mathbb R^d$ is constrained to be small: this is enforced by demanding that $\|\delta\| \le r$, where $\|\cdot\|$ is a specified norm and $r \ge 0$ is the attack budget. One way to measure the performance of a linear model $w \in \mathbb R^d$ under such attacks of size $r$, is via it so-called adversarial risk~\citep{madry2017towards,Xing2021}.
\begin{restatable}{df}{}
For any $w \in \mathbb R^d$ and $r \ge 0$, define the adversarial risk of $w$ at level $r \ge 0$ as follows
\begin{eqnarray}
\label{eq:Ewrdef}
\begin{split}
E^{\|\cdot\|}(w,w_0,r) := % \sup_{T: \|T-I\|_\infty \le r} \mathbb E_{x'\sim T|_{\#}N(0,\Sigma)}(w^\top x' - x^\top w_0)^2\\
%&= 
\mathbb E_{x}\left[\sup_{\|\delta\| \le r}((x+\delta)^\top w - x^\top w_0)^2\right],
\end{split}
\end{eqnarray}
where $x \sim N(0,\Sigma)$ is a random test point.
\end{restatable}
It is clear that $r \mapsto E^{\|\cdot\|}(w,w_0,r)$ is a non-decreasing function and $E^{\|\cdot\|}(w,w_0,0)$ corresponds to the  ordinary risk of $w$, namely
\begin{eqnarray}
\label{eq-ordinary-risk}
E(w,w_0) := \mathbb E_x[(x^\top w - x^\top w_0)^2] = \|w-w_0\|_\Sigma^2.
\end{eqnarray}
% Let $\|w\|_\star := \sup_{\|z\| \le 1}z^\top w$ be the dual of the attacker's norm and $\|w\|_\Sigma := \sqrt{w^\top \Sigma w}$ be the Manhalanobis norm induced by the covariance matrix $\Sigma$.
In classical regression setting, the aim is to find $w$ which minimizes $E(w,w_0)$. In the adversarial setting studied here, the aim is to minimize $E^{\|\cdot\|}(w,w_0,r)$ for any $r \ge 0$.

We will henceforth denote by $E^{\|\cdot\|}_{opt}(w_0,r)$ the smallest possible adversarial risk of a linear model for $\|\cdot\|$-attacks of magnitude $r$, that is
\begin{eqnarray}
\label{eq:Eopt}
E^{\|\cdot\|}_{opt}(w_0,r) := \inf_{w \in \mathbb R^d} E^{\|\cdot\|}(w,w_0,r).
\end{eqnarray}

 We start with the following well-known elementary but useful lemma which proved in the supplemental. Also see \cite{Xing2021,javanmard2022precise} for the special case of Euclidean-norm attacks.
\begin{restatable}{lm}{fench}
Recall that $c_0=\sqrt{2/\pi}$, then for any $w \in \mathbb R^d$ and $r \ge 0$, it holds that
\begin{align}
\label{eq:Ewr}
\begin{aligned}
 E^{\|\cdot\|}(w,w_0,r) &= \|w-w_0\|_{\Sigma}^2+r^2\|w\|_\star^2\\
 &\quad\quad +2c_0 r\|w-w_0\|_\Sigma\|w\|_\star.
 \end{aligned}
\end{align}
\label{lm:fench}
\end{restatable}
The mysterious constant $c_0=\sqrt{2/\pi}$ in Lemma \ref{lm:fench} corresponds to the expected absolute value of a  standard Gaussian random variable. In order to obtain a robust predictor to adversarial attacks, one aims at minimizing the adversarial risk introduced in~\eqref{eq:Ewrdef}. However the objective function of the problem, even in the linear setting~\eqref{eq:Ewr}, is rather complicated to optimize due to its non-convexity.

\section{A Proxy for Adversarial Risk}
The following lemma will be one of the main workhorses in subsequent results, as it allows us to replace the adversarial risk functional $E$ with a more tractable proxy $\widetilde E$.
\begin{restatable}{lm}{sandwich}
\label{lm:sandwich}
For any $w \in \mathbb R^d$ and $r \ge 0$, it holds that
 \begin{eqnarray}
 \label{eq:sandwich0}
 E^{\|\cdot\|}(w,w_0,r) \le \widetilde E^{\|\cdot\|}(w,w_0,r) \le \alpha\cdot E^{\|\cdot\|}(w,w_0,r),
 \end{eqnarray}
 where $\alpha := 2/(1+\sqrt{2/\pi}) \approx 1.1124$~~ and
 \begin{align}
 \label{eq:Etilde}
     \widetilde E^{\|\cdot\|}(w,w_0,r) &:= (\|w-w_0\|_\Sigma+r\|w\|_\star)^2.
 \end{align}
 \end{restatable}

The result is proved in the appendix.
% The following corollary will be one of our main workhorses for obtaining near-optimal robust estimators.
Since $\alpha \approx 1.1124$, the above approximation would allow us to get rougly $90\%$-optimality in the adversarial risk by minimizing the (much simpler) proxy function $w \mapsto\widetilde{E}^{\|\cdot\|} (w,w_0,r)$ instead. This will precisely be the focus of the next sections.
% Also,
% unlike $E(w,r)$ defined previously, (the square root of) $\widetilde E(w,r)$ defines a proper measure of the distance of $w$ from $w_0$. Furthermore,
% because all norms on $\mathbb R^d$ (being a finite-dimensional Banach space) are equivalent, any consistent estimator $\widehat w_n$ of $w_\star$ will have $E(\widehat w_n,r) = o_{n,\mathbb P}(1)$. However, as we shall see, the robustness constraint might slow the rates of convergence, i.e compared to normal learning, more samples are required for robust learning.

We also denote $\widetilde E_{opt}^{\Vert\cdot\Vert}(w_0,r)$ the smallest possible value of the adversarial risk proxy $\widetilde{E}^{\|\cdot\|}(\cdot,w_0,r)$.
% , that is
% \begin{eqnarray}
% \label{eq:Etildeopt}
% \widetilde E_{opt}(w_0,r) := \inf_{w \in \mathbb R^d} \widetilde E(w,w_0,r).
% \end{eqnarray}

\section{Gradient-Descent for Linear Regression}
\label{sec:gd}
% In the classical regression setting, one aims at obtaining the generative model, namely $w_0$ by simply minimizing the standard risk over the parametric family of linear functions that is:
% \begin{align*}
%     \min_{w\in\mathbb{R}^d} E(w,w_0)
% \end{align*}
While the optimization of adversarial risk can be complex, the minimization of ordinary risk can be obtained by a simple gradient-descent. When applying a vanilla gradient-descent (GD) scheme with a step-size $\eta>0$, starting at $w_\eta(0):= 0_d$ to the ordinary risk defined in~\eqref{eq-ordinary-risk}, one obtains at each iteration $t\geq 1$  the following updates:
\begin{eqnarray}
\begin{split}
    w_{\eta}(t,w_0)&:=w_{\eta}(t-1) - \eta \nabla_w E(w_{\eta}(t-1),w_0)\\
    &= (I_d - (I_d - \eta \Sigma)^{t})w_0
    \end{split}
\end{eqnarray}
This scheme can be seen as a discrete approximation of the gradient flow induced by the following ODE:
$$
\left\{
    \begin{array}{ll}
       \dot w(t) =-\Sigma(w(t) - w_0) \\
        w(0)=0_d
    \end{array}
\right.
$$
which has a closed-form solution given by 
\begin{eqnarray}
\label{eq:gdynamics}
    w(t,w_0):=(I_d - \exp(-t\Sigma))w_0.
\end{eqnarray}
Our goal is to evaluate the robustness of the predictors obtained along the gradient descent path. We consider the continuous-time GD, as the analysis of discrete-time GD is analogous due to the absence of noise: the former is an infinitely small step-size $\eta$ limit of the latter ~\citep{continuousGDAnul2019,continuousGDAnul2020}. In the following we mean by early-stopped GD, any predictor (indexed by training time, $t$) obtained along the path of the gradient flow. Observe that for such predictors one has always that
\begin{align*}
 E^{\Vert\cdot\Vert_2}(w(t,w_0),w_0,r)\geq E^{\|\cdot\|}_{opt}(w_0,r) 
\end{align*}
and so for any $t\geq 0$. In particular we will focus on the one that minimizes the adversarial risk at test time, that is $ \inf_{t \geq 0} E^{\Vert\cdot\Vert_2}(w(t,w_0),w_0,r)$. 

% In order to do so, we define for all $r\geq 0$, the adversarial risk with respect to a given norm $\Vert\cdot\Vert$ (the attacker's norm) as follows:
% \begin{align*}
%     E^{\Vert\cdot\Vert}(w,w_0,r):=\mathbb{E}_{x~\sim P_x}\left(\sup_{\Vert\delta\Vert \leq r} (w^T(x+\delta) - w_0^T x)^2\right)
% \end{align*}

\subsection{Euclidean Attacks: Almost-optimal Robustness}
\label{subsec:euclidean}
Here we consider Euclidean attacks, meaning that the attacker's norm is $\Vert\cdot\Vert=\Vert\cdot\Vert_{*}=\Vert\cdot\Vert_2$.
% Let us denote $\{v_1,\dots,v_d\}$ an orthonormal basis where $\Sigma$ can be diagonalized, $\lambda_1\geq \dots\geq \lambda_d$ the eigenvalues of $\Sigma$ and set $c_j:= w_0^\top v_j$ for all $j\in [\![1,d]\!]$. The adversarial risk associated to $w(t)$ in population can be written as
% \begin{align*}
%     E^{\Vert\cdot \Vert_2}(w(t),r)&= E_1(t) + r^2 E_2(t) + 2 r \sqrt{2/\pi} \sqrt{E_1(t)E_2(t)}\\
%     E_1(t)&:= \sum_{j=1}^d (1 - \alpha \lambda_j)^{2t} \lambda_j  c_j^2 ~~\text{and}~~\\
%     E_2(t)&:= \sum_{j=1}^d[1 - (1-\alpha\lambda_j)^{t}]^2c_j^2\; .
% \end{align*}
In the following proposition, we first characterize the non-robustness of the generative model $w_0$.
\begin{restatable}{prop}{generative}
\label{prop:generative}
If $r\leq \sqrt{2/\pi}\frac{\Vert w_0\Vert_2}{\Vert w_0\Vert_{\Sigma^{-1}}}$, then we have that
\begin{align*}
    E^{\Vert\cdot\Vert_2}(w_0,w_0,r)=  E^{\Vert\cdot\Vert_2}_{opt}(w_0,r).
\end{align*}
and as soon as $r>\sqrt{2/\pi}\frac{\Vert w_0\Vert_2}{\Vert w_0\Vert_{\Sigma^{-1}}}$, we have
\begin{align*}
     E^{\Vert\cdot\Vert_2}(w_0,w_0,r) / E^{\Vert\cdot\Vert_2}_{opt}(w_0,r)\geq \frac{r^2\Vert w_0\Vert_2^2}{\Vert w_0\Vert_{\Sigma}^2}.
\end{align*}
\end{restatable}
It is important to notice that the generative model $w_0$ can be optimal w.r.t both the standard risk $E(\cdot,w_0)$ and the adversarial risk $E^{\Vert\cdot\Vert_2}(\cdot,w_0,r)$ for Euclidean attacks as soon as $r$ is sufficiently small. However, as $r$ increases, its adversarial risk becomes arbitrarily large. Therefore, applying a GD scheme until convergence may lead to predictors which are not robust to adversarial attacks even in the Euclidean setting. In the next proposition we investigate the robustness of the predictors obtained along the path of the GD scheme $(w(t,w_0))_{t\geq 0}$ and we show that for any attack $r\geq 0$, this path contains an optimally robust predictor (up to an absolute constant $\beta:=1.6862$).
\begin{restatable}{prop}{GDisOpt}
\label{prop:GDisOpt}
The following hold:

-- If $r\leq \sqrt{2/\pi}\frac{\Vert w_0\Vert_2}{\Vert w_0\Vert_{\Sigma^{-1}}}$ or
$r\geq \sqrt{\pi/2}\frac{\Vert w_0\Vert_{\Sigma^2}}{\Vert w_0\Vert_{\Sigma}}$, then
\begin{align}
     \inf_{t \geq 0}E^{\Vert\cdot\Vert_2}(w(t,w_0),w_0,r)= E^{\Vert\cdot\Vert_2}_{opt}(w_0,r).
\end{align}

-- If $\sqrt{2/\pi}\frac{\Vert w_0\Vert_2}{\Vert w_0\Vert_{\Sigma^{-1}}} < r <\sqrt{\pi/2}\frac{\Vert w_0\Vert_{\Sigma^2}}{\Vert w_0\Vert_{\Sigma}}$, we have that
\begin{align}
     \inf_{t \geq 0}E^{\Vert\cdot\Vert_2}(w(t,w_0),w_0,r)\leq \beta E^{\Vert\cdot\Vert_2}_{opt}(w_0,r)
\end{align}
\end{restatable}
Therefore the early-stopped vanilla GD scheme is able to capture an almost-optimally robust predictor for any Euclidean attack of radius $r$ (see Figure~\ref{fig:l2-case} for an illustration).

In our proof of Proposition \ref{prop:GDisOpt}, we show that GD early-stopped at time $t$ has the same adversarial risk (up to multiplicative constant) as a ridge estimator with regularization parameter $\lambda \propto 1/t$. The result then follows from \citep{Xing2021}, where it was shown the minimizer of the adversarial risk under Euclidean attacks is a ridge estimator.
% In terms of ordinary risk, such reduction from early-stopped GD to ridge regression is now well-established \citep{continuousGDAnul2019,continuousGDAnul2020}.

\begin{figure}[h]
\vspace{-0.1cm}
\centering
\includegraphics[width=0.99\linewidth]{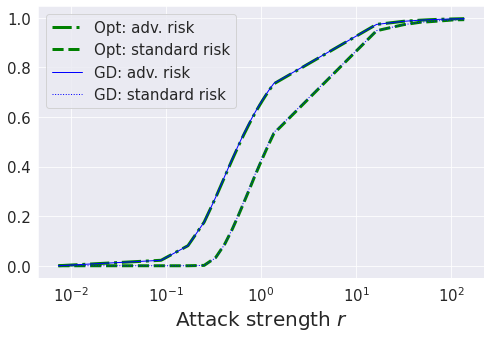}
\caption{We consider a $d=2$ dimensional case where $w_0$ and $\Sigma$ are sampled according to a Gaussian and an Exponential distributions respectively. We plot the adversarial risk under Euclidean attacks of the optimal predictor along the GD path (\textbf{GD: adv. risk}) as well as its standard risk (\textbf{GD: standard risk}) and we compare them to the risks of the optimal predictor (\textbf{Opt}) solving adversarial problem when varying the attack strength $r$.}
\label{fig:l2-case}
\vspace{-0.1cm}
\end{figure}

In the isotropic case, i.e. when $\Sigma=I_d$, early-stopped GD even achieves the exact optimal adversarial risk.
\begin{restatable}{prop}{isoexact}
\label{prop:iso-exact}
Assume that $\Sigma=I_d$, then for all $r\geq 0$, we have $\inf_{t\geq 0} E^{\Vert\cdot\Vert_2}(w(t,w_0),r) =  E^{\Vert\cdot\Vert_2}_{opt}(w_0,r)$.
\end{restatable}

However, such results are only possible when the attacks are Euclidean. In the following section, we show that vanilla GD, along its whole optimization path, can be arbitrarily sub-optimal in terms of adversarial risk for Mahalanobis Attacks.

\subsection{Mahalanobis Attacks: Sub-optimality of Gradient-Descent}
Let us consider attacks w.r.t the Mahalanobis norm induced by a symmetric positive definite matrix $B$, i.e. we consider the case where 
\begin{align*}
    \Vert \cdot \Vert = \Vert \cdot \Vert_{B}:=\Vert B^{1/2} \cdot \Vert_{2}\; .
\end{align*}
In the next proposition, we present a simple case where GD fails to be adversarially robust under such attacks.
\begin{restatable}{prop}{GDisBadforMan}
\label{thm:GDisBadforMan}
Let $d=2$, $\Sigma=I_d$ and for any integer $m \ge 1$, let us consider the following positive-definite matrix
\begin{align}
\label{eq-B-counter-example}
B=B(m)=
\begin{pmatrix}
1/m & 0\\
0 & m
\end{pmatrix}.
\end{align}
Also, consider the following choice of generative model $w_0 = w_0(m) =(1/\sqrt{m},1)$. Then, for any fixed $r>0$,
\begin{align*}
   \underset{m\xrightarrow[]{}+\infty}{\lim}\dfrac{\inf_{t\geq 0}E^{\Vert\cdot\Vert_{B}}(w(t,w_{0}),w_{0},r)}{E^{\Vert\cdot\Vert_{B}}_{opt}(w_{0},r)}=+\infty.
\end{align*}
\end{restatable}
Therefore under Mahalanobis attacks, any predictor obtained along the path of GD can be arbitrarily sub-optimal. To alleviate this issue, we propose in the next section a modified version of the vanilla GD scheme which can handle such attacks and even more general ones.

\section{An Adapted Gradient-Descent: GD+}
\label{sec:gdplus}
In fact, it is possible to obtain an almost optimally robust predictor for Mahalanobis attacks using a modified version of the GD scheme, termed GD+. Let $M\in\mathcal{M}_d(\mathbb{R})$ be an arbitrary invertible matrix. In order to build an almost-optimally robust predictor for such attacks, we propose to apply a GD scheme on a transformed version of the data. More precisely, we propose to apply a GD scheme to the following objective function:
\begin{align*}
    E_{M}(w,w_0):=\mathbb{E}_{(x,y)\sim P_{xy}}(w^\top Mx -y)^2
\end{align*}
which leads to the following optimization dynamics:
\begin{align}
\label{eq:GD+}
    w^{M}(t,w_0):= (I_d - e^{-t M \Sigma M^T})(M^{-1})^{\top}w_0.
\end{align}
% In the following, we first  the robustness of GD+ under Mahalanobis norm-attacks induced by an arbitrary symmetric positive definite matrix $B$.
% \subsection{Optimality to Mahalanobis norm-attacks}
% Here we focus on the robustness of our modified GD (GD+) under Mahalanobis norm-attacks. 
This transformation amounts to apply feature-dependent gradient steps determined by $M$ to the classical GD scheme. In the following proposition, we show that when $M$ is adapted to the attack, early-stopped GD+ is optimally robust (up to an absolute constant).
\begin{restatable}{prop}{GDPlusIsGood}
\label{prop:GDPlusIsGood}
For any $B \in \mathcal{S}_{d}^{++}(\mathbb{R})$ and $r\geq 0$,
$$\inf_{t \geq 0}E^{\Vert\cdot\Vert_{B}}(B^{1/2} w^{B^{1/2}}(t,w_0),r)\leq \beta E^{\Vert\cdot\Vert_{B}}_{opt}(r,w_0).$$
\end{restatable}
% \begin{proof}
% In fact we know that a GD on data of covariance $\Sigma$ and model $w_0$ is optimal for $\ell_2$ attacks up to $\beta$. Now remarks that by our changement of variable, we consider a GD on data of covariance $M\Sigma M^{\top}$ and model $(M^{\top})^{-1}w_0$. Then we obtain that our modified GD minimize 
% \begin{align*}
%     &\Vert w - (M^{T})^{-1}w_0\Vert_{M\Sigma M{\top}}^2 + r^2\Vert w\Vert_2^2 \\
%     &+ 2rc_0\Vert w - (M^{T})^{-1}w_0\Vert_{M\Sigma M{\top}}\Vert w\Vert_2
% \end{align*}
% however we have that:
% \begin{align*}
%   \Vert w - (M^{T})^{-1}w_0\Vert_{M\Sigma M{\top}}^2 =   \Vert M^{\top}w - w_0\Vert_{\Sigma}^2 
% \end{align*}
% and therefore, by applying again a changement of variables, we obtain that $M^{\top}w^{M}$ minimize  
% \begin{align*}
%     &\Vert w -w_0\Vert_{\Sigma}^2 + r^2\Vert (M^{T})^{-1}w\Vert_2^2 \\
%     &+ 2rc_0\Vert w - w_0\Vert_{\Sigma}\Vert  (M^{T})^{-1}\Vert_2
% \end{align*}
% Taking finally $M=B^{1/2}$ shows the result as the dual norm of $\Vert\cdot\Vert_B$ is in fact $\Vert\cdot\Vert_{B^{-1}}$.
% \end{proof}
Therefore by choosing $M=B^{1/2}$, GD+ is able to obtain near-optimality under $\Vert\cdot\Vert_{B}$-norm attack. See Figure~\ref{fig:counter-example-1} for an illustration. Note that when $B=B^{1/2}=I_d$, then $w^{I_d}(t,w_0)=w(t,w_0)$, $\Vert\cdot \Vert_{I_d} = \Vert \cdot \Vert_{2}$, and we recover as a special case our result obtained in Proposition~\ref{prop:GDisOpt}. 

% \subsection{Robustness to Mahalanobis-norm-attacks}
% Another case of interest in when $M=B^{-1/2} \Sigma B^{-1/2}$ where $B\in\mathcal{S}_{d}^{++}(\mathbb{R})$. In that case we obtain that 
% $A_{M}=B^{-1/2}$ and therefore our proposed method is optimally robust (up to an absolute constant) w.r.t to attacks in $\Vert \cdot \Vert_{B}$ for any matrix $B\in \mathcal{S}_{d}^{++}(\mathbb{R})$.

% \begin{corollary}
% For any $B \in \mathcal{S}_{d}^{++}(\mathbb{R}^d)$ and $r\geq 0$, we have
% \begin{align*}
%      \inf_{t \geq 0}E^{\Vert\cdot\Vert_{B}}(B^{-1/2} w^{B^{-1/2} \Sigma B^{-1/2}}_c(t),r)\leq 1.6862 E^{\Vert\cdot\Vert_{B}}_{opt}(r).
% \end{align*}
% \end{corollary}

% Another case of interest is when $M$ is an element of the commutant of $\Sigma$, meaning that $\SigmaM=M\Sigma$, then $A_{\Sigma M} = M^{1/2}$, and therefore our proposed GD scheme becomes almost optimally-robust w.r.t to attacks in $\Vert \cdot \Vert_{M}:=\Vert M^{1/2} \cdot \Vert_{2}$. Note also that a direct consequence of what precedes is when $\Sigma=I_d$, then our method is able to find an optimally-robust predictor w.r.t to attacks in $\Vert \cdot \Vert_{M}$ for any matrix $M$ definite positive, that is any Mahalanobis distance.
\begin{figure}[h!]
\vspace{-0.1cm}
\centering
\includegraphics[width=0.99\linewidth]{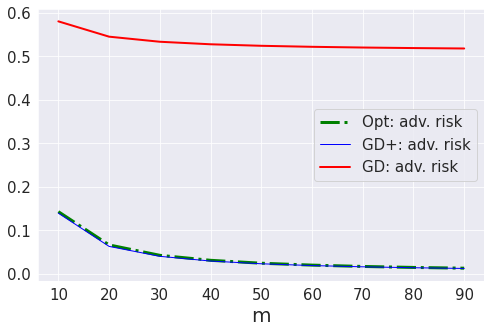}
\caption{Here we plot the example studied in Proposition~\ref{thm:GDisBadforMan} for a fixed radius $r=1$ when varying $m$. \textbf{GD+} represents the modified GD scheme with $M=B^{1/2}$ where $B$ is defined in Eq.~(\ref{eq-B-counter-example}),  \textbf{GD} represent the vanilla GD and \textbf{Opt} is the optimal predictor. Observe that the optimal adversarial risk goes to $0$ as $m$ goes to $+\infty$ while the adversarial risk of the vanilla GD converges towards a constant.}
\label{fig:counter-example-1}
\vspace{-0.1cm}
\end{figure}

In the following section, we investigate the robustness of GD+ scheme on more general attacks. 

\subsection{Robustness Under General Attacks}
\label{subsec:general_attacks}
Our goal here is to provide a simple control of the adversarial risk of GD+ under general norm-attack $\Vert\cdot\Vert$. For that purpose, define the \emph{condition number} $\kappa^{\|\cdot\|}(M)$ of any matrix $M$, w.r.t to the attacker's norm $\|\cdot\|$ as
\begin{align*}
    \kappa^{\Vert\cdot\Vert}(M)&:=\lambda_{max}^{\Vert\cdot\Vert}(M)/\lambda_{min}^{\Vert\cdot\Vert}(M),\,\text{where }\\
    \lambda_{max}^{\Vert\cdot\Vert}(M)&:=\sup_{w\neq 0}\frac{\Vert M w\Vert_2}{\Vert w\Vert_*},\text{ and}\\
    \lambda_{min}^{\Vert\cdot\Vert}(M)&:=\inf_{w\neq 0}\frac{\Vert M w\Vert_2}{\Vert w\Vert_*}
\end{align*}
% then we show the following Proposition.
where dual norm $\Vert \cdot \Vert_*$ is defined with respect to the attacker-norm $\|\cdot\|$. We are now ready to show a general upper-bound of our modified GD scheme.
\begin{restatable}{prop}{geneupper}
\label{prop:gene-upper}
For any $r\geq 0$, and invertible matrix $M\in\mathcal{M}_d(\mathbb{R})$, it holds that
%\begin{align*}
$$\dfrac{\inf_{t \geq 0}E^{\Vert\cdot\Vert}(M^{\top}w^{M}(t,w_0),w_0,r)}{E^{\Vert\cdot\Vert}_{opt}(r,w_0)}\leq \beta \kappa^{\Vert\cdot\Vert}\left((M^{\top})^{-1}\right)^2.
     $$
%\end{align*}
\end{restatable}
% \begin{proof}
% Indeed recall that our modified GD scheme minimize $V_{\Vert (M^{T})^{-1}\cdot\Vert,r,w_0,\Sigma}(w)$ defined as:
% \begin{align*}
%     &\Vert w -w_0\Vert_{\Sigma}^2 + r^2\Vert (M^{T})^{-1}w\Vert_2^2 \\
%     &+ 2rc_0\Vert w - w_0\Vert_{\Sigma}\Vert  (M^{T})^{-1}\Vert_2
% \end{align*}
% Now observe that for all $w$
% \begin{align*}
%   \lambda_{min}^{\Vert\cdot\Vert}((M^{T})^{-1}) \Vert w\Vert_{*}\leq \Vert (M^{T})^{-1}w\Vert_2 \leq  \lambda_{max}^{\Vert\cdot\Vert}((M^{T})^{-1}) \Vert w\Vert_{*}
% \end{align*}
% Therefore we obtain that
% \begin{align*}
%     \inf_t  V_{\Vert \cdot\Vert_{*},r,w_0,\Sigma}(w(t))
%     &\leq  \inf_t V_{\Vert (M^{T})^{-1}\cdot\Vert,r/\lambda_{min}^{\Vert\cdot\Vert}((M^{T})^{-1}),w_0,\Sigma}(w)\\
%     &\leq \beta  \inf_w V_{\Vert (M^{T})^{-1}\cdot\Vert,r/\lambda_{min}^{\Vert\cdot\Vert}((M^{T})^{-1}),w_0,\Sigma}(w)\\
%     &\leq \beta \inf_w  V_{\Vert \cdot\Vert_{*},r\kappa^{\Vert\cdot\Vert}(((M^{T})^{-1})),w_0,\Sigma}(w)\\
%     &\leq \beta\kappa^{\Vert\cdot\Vert}(((M^{T})^{-1}))^2 \inf_w  V_{\Vert \cdot\Vert_{*},r,w_0,\Sigma}(w)
% \end{align*}
% \end{proof}
Therefore along the path of GD+ induced by $M$, one can find a $\beta \kappa^{\Vert\cdot\Vert}((M^{\top})^{-1})^2$-optimally robust predictor against $\Vert\cdot\Vert$-attack. In particular, observe that when $M=B^{1/2}$, and $\Vert\cdot\Vert=\Vert\cdot\Vert_{B}$, we obtain that $\kappa^{\Vert\cdot\Vert}((M^{\top})^{-1})=1$ and we recover as a special case the result of Proposition~\ref{prop:GDPlusIsGood}.

\begin{restatable}{rmk}{}
It is important to notice that $M$ can be chosen arbitrarily, and therefore adapted to the norm of the attacks such that $M\to \kappa^{\Vert\cdot\Vert}((M^{\top})^{-1})$ is minimized. However, minimizing this quantity in general is hard due to the arbitrary choice of the norm $\Vert\cdot\Vert$. 
\end{restatable}

In the next section, we study a specific case of GD+ and provide a sufficient condition on the model $w_0$ such that this scheme is near-optimal under general norm-attacks.

\subsection{A Sufficient Condition for Optimality}
\label{subsec:sufficientGD}
We consider the general case of an arbitrary norm-attack $\|\cdot\|$ with dual $\|\cdot\|_\star$ and we focus on a very specific path induced GD+, which is  when $M=\Sigma^{-1/2}$. In that case, data are normalized and the path drawn by $(M w^{M}(t,w_0))_{t\geq 0}$ is in fact a uniform shrinkage of the generative model. More precisely, the predictors obtained along such a path are exactly the one in the chord $[0,w_0]:=\left\{\gamma w_0 \mid \gamma\in[0,1]\right\}$. In particular, the optimal adversarial risk achieved by this modified GD scheme is given by
\begin{eqnarray}
\begin{split}
    &\inf_{t\geq 0}E^{\Vert\cdot\Vert}(\Sigma^{-1/2} w^{\Sigma^{-1/2}}(t,w_0),w_0,r)\\&\quad\quad=\inf_{\gamma\in[0,1]} E^{\Vert\cdot\Vert}(\gamma w_0,w_0,r)
    % =:E_{shrink}^{\Vert\cdot\Vert}(w_0,r).
    \end{split}
 \end{eqnarray}

Let $g(w_0) \in \mathbb R^d$ be a subgradient of $\Vert\cdot\Vert_\star$ at $w_0$. For example, in the case of $\ell_\infty$-norm-attacks, one may take $g(w_0) = (\mbox{\sign}(w_{0,1}),\ldots,\mbox{sign}(w_{0,d}))$, with $\mbox{sign}(0) := 0$.
In the case of of a Mahalanobis attack where $\|\cdot\| = \|\cdot\|_B$ for some positive-definite matrix $B$, one can take $g(w_0) =B^{1/2}w_0/\|w_0\|_B$ with $g(0) = 0$. 
% For example, in the case of $\ell_\infty$-norm-attacks, we have $\lambda \geq 1/ \sqrt{\text{Trace}(\Sigma^{-1})}$; In the case of euclidean attacks, we have $\lambda\geq\sqrt{\lambda_{min}^{\Vert\cdot\Vert_2}(\Sigma)}$ and more gerenarlly, for Mahalanobis attacks where $\|\cdot\| = \|\cdot\|_B$ with $B$ a positive-definite matrix, we have that $\lambda\geq \sqrt{\lambda_{min}^{\Vert\cdot\Vert_2}(B^{1/2}\Sigma^{1/2}B^{1/2})}$ can be taken.
We can now state our sufficient condition for near-optimality of GD+.
\begin{restatable}{cond}{}
The subgradient $g(w_0) \in \mathbb R^d $ can be chosen such that  
\begin{align*}
    \dfrac{\|g(w_0)\| \|w_0\|_*}{\|g(w_0)\|_{\Sigma^{-1}} \|w_0\|_{\Sigma}}\geq c,
\end{align*}
 where   $c$ is a positive absolute constant. 
\label{cond:dense}
\end{restatable}

The above condition is sufficient in order to obtain near-optimality of GD+ as we show in the next proposition~(see Figure~\ref{fig:sufficient} for an illustration).
\begin{restatable}{thm}{GDrocksdeterministic}
Suppose Condition \ref{cond:dense} is in order. Then, for any positive $r$, it holds for $M=\Sigma^{-1}$ that
\begin{eqnarray*}
\begin{split}
\dfrac{\inf_{t \geq 0}E^{\Vert\cdot\Vert}(M^{\top}w^{M}(t,w_0),w_0,r)}{E^{\Vert\cdot\Vert}_{opt}(r,w_0)}\le (1 \lor 1/c^2)\alpha.
\end{split}
\end{eqnarray*}
\label{thm:GDrocksdeterministic}
\end{restatable}
In particular, for the case of $\ell_\infty$-norm-attacks, we have the following corollary.
\begin{restatable}{cor}{}
Consider the case where $\ell_\infty$-norm-attacks. If there exists an absolute constant $c>0$ such that $\|w_0\|_1 \ge c \sqrt{d}\|w_0\|_2$, then with $M=\Sigma^{-1}$ it holds that 
\begin{eqnarray*}
\begin{split}
\dfrac{\inf_{t \geq 0}E^{\Vert\cdot\Vert_\infty}(M^{\top}w^{M}(t,w_0),w_0,r)}{E^{\Vert\cdot\Vert_\infty}_{opt}(r,w_0)}  \le \left(1 \lor \frac{\kappa^{\Vert\cdot\Vert_2}(\Sigma)}{c^2}\right)\alpha.
\end{split}
\end{eqnarray*}
\label{coro-sufficient}
\end{restatable}
For example, when $w_0=(1,\ldots,1)$ --or equivalently, random $w_0 \sim N(0,I_d)$, one can take $c=1$, and observe that $\|w_0\|_1 \gtrsim d$, $\|w_0\|_2 \asymp \sqrt{d}$, and so the bound in Corollary \ref{coro-sufficient} holds.

\begin{figure}[t!]
\centering
\includegraphics[width=0.99\linewidth]{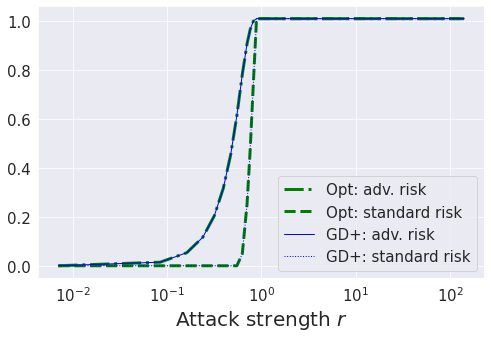}
\caption{We consider the case exhibited in the Corollary~\ref{coro-sufficient} when $d=2$, $\Sigma=I_d$ and $w_0=(1,1)$ under $\ell_{\infty}$-attacks. We plot the adversarial as well as the standard risks for both the optimal shrinkage predictor as well as the optimal predictor of the adversarial risk when varying $r$.}
\label{fig:sufficient}
\vspace{-0.1cm}
\end{figure}
The  Condition~\ref{cond:dense} that ensures optimality of GD+ in Theorem \ref{coro-sufficient} cannot be removed. Indeed, we exhibit a simple case where the uniform shrinkage strategy fails miserably to find robust models even when they exist. 
\begin{restatable}{prop}{bangbang}
Let $\Sigma = I_d$, then it is possible to construct $w_0 \in \mathbb R^d$ and $r > 0$ such that in the limit $d \to \infty$, it holds that $r \to 0,\,r\sqrt d \to +\infty$, and
\begin{align*}
 \dfrac{\inf_{t \geq 0}E^{\Vert\cdot\Vert_\infty}( w (t,w_0),w_0,r)}{E^{\Vert\cdot\Vert_\infty}_{opt}(r,w_0)} \to +\infty.
\end{align*}
\label{thm:bangbang}
\end{restatable}
Therefore the uniform shrinkage strategy induced by GD+ (which here reduces to vanilla GD since $\Sigma=I_d$) is not adapted for all scenarios; it may fail to find a robust model even when one exists. 

In the next section, we restrict ourselves to the case of $\ell_{p}$-norm-attacks for $p\in[1,+\infty]$ and show that GD+ is able to reach optimal robustness as soon as the data has uncorrelated features.

\subsection{Optimality Under $\ell_{p}$-norm Attacks}
Now, let $p\in[1,+\infty]$ and consider attacks w.r.t the  
 $\ell_p$-norm, i.e the attack strength is measured w.r.t the norm $\Vert\cdot\Vert=\Vert\cdot\Vert_{p}$, with dual norm $\Vert\cdot\Vert_\star=\Vert\cdot\Vert_{q}$, where $q \in [1,+\infty]$ is the harmonic conjugate of $p$. Popular examples in the literature are $p=2$ (corresponding to Euclidean attacks, considered in Section \ref{subsec:euclidean}) and $p=\infty$. In this section we assume that $\Sigma$ is a diagonal positive-definite matrix. This assumption translates the fact that, both norm $\Vert\cdot\Vert_{\Sigma}$ and $\Vert\cdot\Vert_q$ act on the same coordinates system. When these two norms are aligned, we show in the next theorem that the the minimiser of the proxy introduced in Eq.~(\ref{eq:Etilde}) is in fact a non-uniform shrinkage of $w_0$ which can be recovered by GD+. An illustration of the result is provided in Figure~\ref{fig:l1-attack-gene}.
% under this assumption we show in the next Theorem that in fact our modified GD scheme is able to recover this non-uniform shrinkage, and therefore be near-optimal.
\begin{restatable}{thm}{lpoptimal}
\label{thm-lp-optimal}
Let $\Sigma$ be any definite positive diagonal matrix and $p\in [1,+\infty]$, then we have
\begin{align*}
    \inf_{M\in\mathcal{M}_d(\mathbb{R}),t\geq 0} \frac{E^{\Vert\cdot\Vert_p}(M^{\top}w^{M}(t,w_0),w_0,r)}{E^{\Vert\cdot\Vert_p}_{opt}(r,w_0)}\leq \alpha.
\end{align*}
\end{restatable}
Therefore GD+ is able to reach near-optimality in term of adversarial risk, and so for any $\ell_p$-attacks with $p\in[1,+\infty]$ as soon as $\Vert\cdot\Vert_p$ and $\Vert\cdot\Vert_{\Sigma}$ are aligned. 

\begin{figure}[h!]
\vspace{-0.1cm}
\centering
\includegraphics[width=0.99\linewidth]{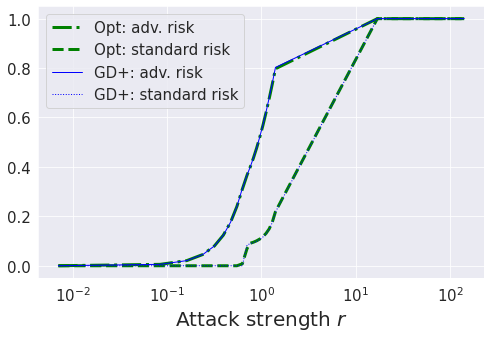}
\caption{We consider the case where $d=2$, $\Sigma$ is diagonal and  $w_0$ and the diagonal coefficients of $\Sigma$ are sampled according to a Gaussian and an exponential distribution respectively. The choice of $M$ is obtained by fine-tuning the GD+. We compare the adversarial and standard risks under $\ell_{\infty}$ of the optimal predictor in the GD+ path with those of the optimally robust predictor when varying the radius $r$.}
\label{fig:l1-attack-gene}
\vspace{-0.1cm}
\end{figure}

Applying a GD+ scheme in practice might be difficult for general attacks, as the choice of the transformation $M$ must be adapted accordingly. To alleviate this issue, we propose in the next section a simple and tractable two-stage estimator able to reach near-optimality and so for general attack.

\section{Efficient Algorithms for Attacks in General Norms}
\label{sec:algorithmic}
We propose a simple tractable estimator whose adversarial risk is optimal up to within a multiplicative constant of 1.1124. Here, we drop the assumption of infinite training data $n=\infty$. Thus, the estimators are functions of the finite-training dataset $D_n:=\{(x_1,y_1),\ldots,(x_n,y_n)\}$, generated according to \eqref{eq:generative}.
% Second we consider recently propose definition of adversarial risk under which the generative model has optimal adversarial risk, and derive the minimax optimal adversarial risk.
% Our algorithm works for any covariance matrix $\Sigma_x$ and attacker norm with tractable proximal operator.

\subsection{A Two-Stage Estimator and its Statistical Analysis}
Consider any vector $\widehat w$ which minimizes the adversarial risk proxy $w \mapsto \widetilde E^{\Vert\cdot\Vert}(w,w_0,r)$ defined in \eqref{eq:Etilde}. Note apart from its clear dependence on the generative model $w_0$, $\widetilde E^{\Vert\cdot\Vert}$ also depends on the feature covariance matrix $\Sigma$. However, we don't assume that $w_0$ nor $\Sigma$ are known before hand; it has to be estimated from the finite training dataset $\mathcal D_n$.  Thus, we propose a two-stage estimator  described below in Algorithm \ref{alg:twostage}.
\begin{algorithm}
\caption{Proposed Two-Stage Estimator.}
\label{alg:twostage}
\begin{algorithmic}[1]
    \State \textbf{Stage 1:} Compute consistent estimators $\widehat w_0$ and $\widehat \Sigma$ from $w_0$ and $\Sigma$ respectively, from the data $\mathcal D_n$.
    \State \textbf{Stage 2:} Compute $\widehat w$ which minimizes the adversarial risk proxy $w \mapsto \widetilde{E}^{\Vert\cdot\Vert}(w,w_0,r)$ defined in \eqref{eq:Etilde}. See Algorithms \ref{alg:pd} and \ref{alg:st} for implementations of this step.
\end{algorithmic}
\end{algorithm}

\textbf{Stage 1} of Algorithm \ref{alg:twostage} can be implemented using off-the-shelf estimators which adapt to the structural assumptions on $\Sigma$ and $w_0$ and $\Sigma$ (sparsity, etc.). Later, we will provide simple tractable algorithms for implementing \textbf{Stage 2}. Note that \textbf{Stage 2} implicitly requires the knowledge of the attacker-norm as it aims at minimizing $\widetilde{E}^{\Vert\cdot\Vert}$.

\begin{figure}
    \centering
    \includegraphics[width=.48\textwidth]{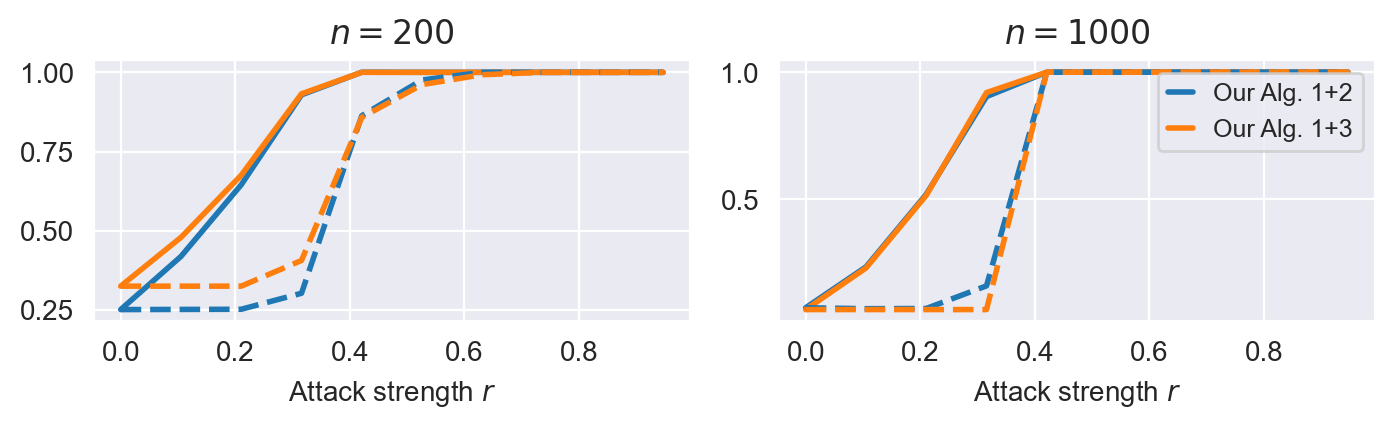}
    \caption{Experiments with our two-stage estimator Algorithm \ref{alg:twostage}. Here, we focus on  $\ell_\infty$-norm-attacks. "1+2" means \textbf{Stage 2} of our Algorithm \ref{alg:twostage} is computed via the primal-dual algorithm (Algorithm \ref{alg:pd}), while "1+3" means it is computed via the Algorithm \ref{alg:st}. The experimental setting here is: input-dimension $d=200$; covariance matrix $\Sigma$ of the features = heterogeneous diagonal, with entries from $\mathrm{Exp}(1)$ distribution; the generative $w_0$ is $s$-sparse vector (with $s=10$), normalized so that $\|w_0\|_\Sigma^2 = 1$.
    Notice how the adversarial risk of the estimator improves with number of samples $n$, as expected. See supplemental for details.
    }
    \label{fig:algos}
\end{figure}

\subsection{Consistency of Proposed Two-Stage Estimator}
We now establish the consistency of our proposed two-stage estimator $\widehat w$ computed by Algorithm \ref{alg:twostage}.
Let $\widehat \Sigma$ be an operator-norm consistent data-driven minimax estimator for $\Sigma$, and %, over some prior class $\mathfrak S$ of covariance matrices.
% We will assume that we're either (1) in low-dimensional regime $n \gg d$ so that we can take $\widehat \Sigma = X^\top X/n$, or (2) $\Sigma$ is known to us, and so we can take $\widehat \Sigma = \Sigma$.
let $\widehat w_0$ be a consistent estimator of the generative model $w_0$.
% , over some prior class $\mathcal C_0$ for $w_0$ (e.g $s$-sparse vectors in $\mathbb R^d$).
Define error terms
\begin{eqnarray}
e_1 := % \sup_{w_0 \in \mathcal C_0}
\|\widehat w_0-w_0\|_2,\, e_2 := % \inf_{\Sigma \in \mathfrak S}
\|\widehat \Sigma-\Sigma\|_{op}.%\, e_3 := \sup_{w_0 \in \mathcal C_0}\|\widehat w_0-w_0\|_\star. 
\end{eqnarray} 
% Consistency of the estimators $\widehat w_0$ and $\widehat \Sigma$ means that $e_1,e_2 = o_{n,\mathbb P}(1)$. 
We are now ready to state the adversarial risk-consistency result for our proposed two-stage estimator (Algorithm \ref{alg:twostage}).
\begin{restatable}{thm}{stats}
\label{thm:stats}
For all $r \ge 0$, it holds that
\begin{align*}
\widetilde E_{opt}^{\Vert\cdot\Vert}(w_0,r) &\le \widetilde E^{\Vert\cdot\Vert}(\widehat w,w_0,r) \le \widetilde E_{opt}^{\Vert\cdot\Vert}(w_0,r) + \Delta,\\
E_{opt}^{\Vert\cdot\Vert}(w_0,r) &\le E^{\Vert\cdot\Vert}(\widehat w,w_0,r) \le \alpha E_{opt}^{\Vert\cdot\Vert}(w_0,r)+\Delta,
\end{align*}
where $\alpha := 2/(1+c_0) \approx 1.1124$, $\Delta = O(e_1^2 + e_2^2)$, where the hidden constant in the big-O is of order $\max(\|\Sigma\|_{op}^2, \|w_0\|_\Sigma^2)$.
\end{restatable}
% Note that in the above theorem, the second inequality follows from the first, thanks to Lemma \ref{lm:sandwich}.
Thus, our proposed two-stage estimator is robust-optimal up to within a multiplicative factor $2/(1+c_0)\approx 1.1124$, and an additive term $O(e_1^2+e_2^2)$ which is due to estimation error of $w_0$ and $\Sigma$ from training data.
Note that if we assume that the covariance matrix $\Sigma$ of the features is known, or equivalently that we have access to an unlimited supply of unlabelled data from $N(0,\Sigma)$, then we effectively have $e_2 = 0$. In this case, the statistical error of our proposed estimator $\widehat w$ is dominated by the error $e_1^2$ associated with estimating the generative model $w_0$. Under sparsity assumptions, this error term is further bounded by $\dfrac{\sigma_\epsilon^2s\log(ed/s)}{n}$ (thanks to the following well-known result), which tends to zero if the input dimension $d$ does not grow much faster than the sample size $n$.
.
\begin{restatable}[\citep{BickelSimultaneous2009,BellecSlopeMeets2018}]{prop}{}
If $1 \le s \le d/2$, then under some mild technical conditions, it holds w.h.p that
\begin{eqnarray}
\inf_{\widehat w_0}\sup_{w_0 \in B_0^d(s)} \underbrace{\|\widehat w_0-w_0\|_2}_{e_1} \asymp \sigma_\epsilon\sqrt{\frac{s\log(ed/s)}{n}}.
\end{eqnarray}
where $ B_0^d(s)$ is defined in Eq.~(\ref{eq-sparse-ball}). Moreover, the above minimax bound is attained by the square-root Lasso estimator with tuning parameter $\lambda$ given by
%\begin{eqnarray}
%\label{eq:lambdaoptlasso}
$\lambda \asymp \sqrt{\dfrac{\log(2d/s)}{n}}$.
% \end{eqnarray}
\label{prop:lassobellec}
\end{restatable}
% Thus, the generative model $w_0$ can be consistently estimated from training data under sparsity assumptions and if the input-dimension does grow too fast.

\begin{restatable}{rmk}{}
Note that, in the special case of Euclidean-norm attacks, our Theorem \ref{thm:stats} (which works for all attack norms) recovers the adversarial risk-consistency result established in \citep{Xing2021} as a special case.
\end{restatable}
% \textcolor{red}{Elvis: Include a table with precise upper-bounds under sparsity assumptions, etc.}

\subsection{Algorithm 1: Primal-Dual Algorithm}
We now device a simple primal-dual algorithm for computing the second stage of our proposed estimator (Algorithm \ref{alg:twostage}). The algorithm works for any covariance matrix $\Sigma$ and norm-attack with tractable proximal operator.

Let $\widehat w_0$ and $\widehat \Sigma$ be the estimates computed in the \textbf{Stage 1} of Algorithm \ref{alg:twostage}.
Define and $K := \widehat \Sigma^{1/2}$,
$a := K\widehat {w_0}$,
% $w^1 := \Sigma^{-1/2} w_0$,
$f(z) := \|z-a\|_2$, $g(w) := r\|w\|_\star$. Recall now that 
$\sqrt{\widetilde{E}^{\Vert\cdot\Vert}(w,w_0,r)} =  \|w-w_0\|_{\Sigma} + r\|w\|_\star$ and so for any model $w$. Then by "dualizing", we get that 
\begin{eqnarray}
\begin{split}
&\inf_{w \in \mathbb R^d} \|w-w_0\|_{\Sigma} + r\|w\|_\star= \inf_{w \in \mathbb R^d} f(Kw) + g(w)\\
&= \inf_{w \in \mathbb R^d} g(w) + \sup_{z \in \mathbb R^d} z^\top K w - f^\star(z)\\
&= \inf_{w \in \mathbb R^d}\sup_{z \in \mathbb R^d} \underbrace{z^\top K w - f^\star(z) + g(w)}_{H(w,z)},
\end{split}
\end{eqnarray}
where $f^\star$ is the Fenchel-Legendre transform of $f$.
Consider the following so-called Chambolle-Pock algorithm \citep{ChambollePock2010} for computing a saddle-point for the function $H$.

\begin{algorithm}
\caption{Primal-Dual algorithm which implements \textbf{Stage 2} of Algorithm \ref{alg:twostage}. Only one iteration is shown here.}
\label{alg:pd}
\textbf{Inputs: } $\widehat w_0, \widehat \Sigma,\eta_1,\eta_2,z^{(0)}=w^{(0)}=\mathbf{0}_d.$
\begin{algorithmic}[1]
\State $z^{(t+1)} \leftarrow  \mathrm{proj}_{B_{\|\cdot\|_2}^d}(z^{(t)} + \eta_2 \widehat\Sigma^{1/2} (u^{(t)}- \widehat w_0))$
\State $w^{(t+1)} \leftarrow \mathrm{prox}_{\eta_1r\|\cdot\|_\star}(w^{(t)}-\eta_1 \widehat\Sigma^{1/2} z^{(t+1)})$,
\State $u^{(t+1)} \leftarrow 2w^{(t+1)}-w^{(t)}$
\end{algorithmic}
\end{algorithm}
% Here, $\mbox{proj}_C:\mathbb R^d \to C$ is the orthogonal projector unto the subset $\mathbb R^d$, and $\mbox{prox}_h:\mathbb R^d \to \mathbb R^d$ is the proximal operator of the function $h:\mathbb R^d \to \mathbb R$ (See Appendix xyz for formal definition).
Here, the $\eta_k$'s are stepsizes chosen such that $\eta_1\eta_2 \|\widehat\Sigma\|_{op}^{1/2} < 1$.
The nice thing here is that the projection onto the $\ell_2$-ball (line 1) admits a simple analytic formula.
In the case of $\ell_\infty$-norm-attacks, the second line corresponds to the well-known soft-thresholding operator; etc. Refer to Figures \ref{fig:algos}~\&~\ref{fig:comparison} for empirical illustrations of the algorithm. The following convergence result follows directly from \citep{ChambollePock2010}.
\begin{restatable}{prop}{}
\label{prop:pd}
Algorithm \ref{alg:pd} converges to a stationary point $(w^{(\infty)},z^{(\infty)})$ of the $H$ at an ergodic rate $O(1/t)$.
\end{restatable}

\begin{figure}
\vspace{-0.1cm}
\centering
\includegraphics[width=1\linewidth]{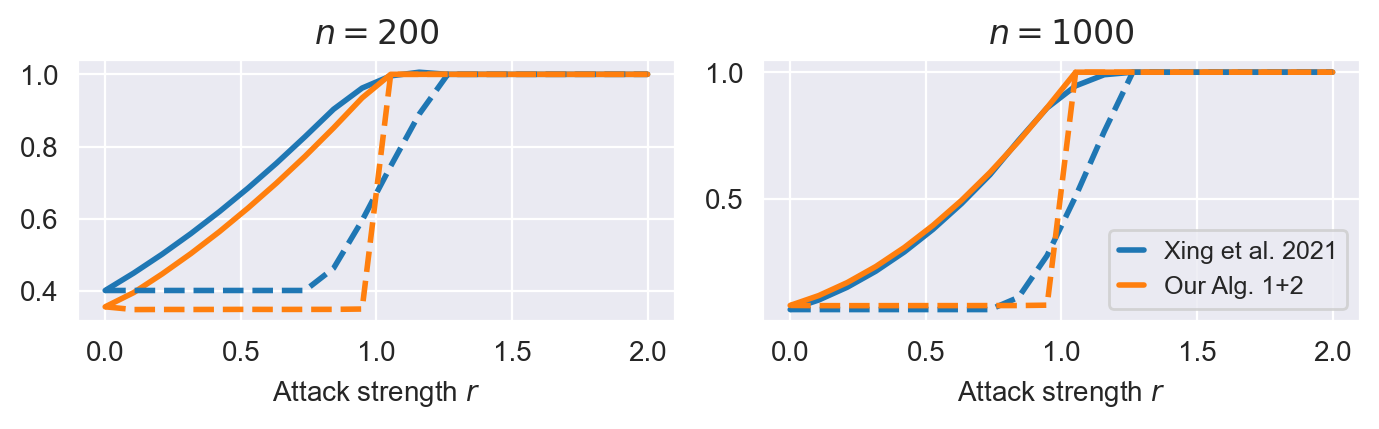}
\caption{We compare our proposed two-stage estimator given in Alg. \ref{alg:twostage} with the one proposed in~\cite{Xing2021}, in the specific setting of Euclidean attacks when varying sample size $n$ and attack strength $r$. Here, \textbf{Stage 2} of our two-stage estimator is computed via Alg. \ref{alg:pd}. Dashed and plain lines respectively represent the standard and adversarial risks.  We see that though more general, our algorithm is able to recover similar performances as \cite{Xing2021} in the special case of Euclidean attacks. %See appendix for further details.
}
\label{fig:comparison}
\vspace{-0.1cm}
\end{figure}

\subsection{Algorithm 2: Simple Thresholding-Based Algorithm in the Case of $\ell_\infty$-Norm Attacks}
Now, suppose the covariance matrix of the features $\Sigma$ is diagonal, i.e $\Sigma = \diag(\lambda_1,\ldots,\lambda_d)$. In the case of $\ell_\infty$-attacks, we provide a much simpler and faster algorithm for computing the second stage of our proposed estimator. Indeed, for $\ell_{\infty}$-attack we obtain an explicit form of the optimal solution minimizing Eq.~(\ref{eq:Etilde}) as shown in the next Proposition.
% Indeed, applying the previous Lemma \ref{lm:subdifferential} with $a=w_0$ and $g:=r\|\cdot\|_1$, we obtain that the minimizer of $w \mapsto \|w-w_0\|_\Sigma + r\|w\|_1$ is of the form $w(t)$, for some $t \in [0,\infty]$, where
% \begin{eqnarray}
% w(t) := (I+rt\Sigma^{-1} \partial \|\cdot\|_1)^{-1}(w_0),
% \end{eqnarray}
% i.e
% $(w_\star)_j:=\mbox{ST}((w_0)_j; rt/\lambda_j)$ for all $j \in [d]$,
% and $\mbox{ST}(\cdot;s)$ is the \emph{soft-thresholding} operator at level $s$, defined component-wise by
% \begin{eqnarray}
% \mbox{ST}(u; s) = (|u|-s)_+\mbox{sign}(u).
% \end{eqnarray}
% The interesting thing here is that $\mbox{ST}(u; s)=0$ as soon as $s \ge |u|$. Define the scalar $c=c(w_0,\Sigma)$ by
% \begin{eqnarray}
% c := \max_{1 \le j \le d}|(w_0)_j| \lambda_j.
% \end{eqnarray}
We deduce the following result.
\begin{restatable}{prop}{}
Let $c := \max_{1 \le j \le d}|(w_0)_j| \lambda_j$. There exists $t \in [0,c]$ such that $w(t) \in \mathbb R^d$ is  a minimizer of the convex function $w \mapsto \|w-w_0\|_\Sigma + r\|w\|_1$, where
\begin{eqnarray}
w(t)_j = \mbox{ST}((w_0)_j;rt/\lambda_j),\text{ for all }j \in [d].
\label{eq:STsol}
\end{eqnarray}
and $\mbox{ST}(\cdot;s)$ is the soft-thresholding operator at level $s$.
\end{restatable}
An inspection of \eqref{eq:STsol} reveals a kind of feature-selection. Indeed, if the $j$th component of the ground-truth model $w_0$ is small in the sense that $|(w_0)_j| \le rt/\lambda_j$, then $w(t)_j = 0$, i.e the $j$th component of $w_0$ should be suppressed. On the other hand, if $|(w_0)_j| \ge rt/\lambda_j$, then $(w_0)_j$ should be replaced by the translated version $(w_0)_j - \mathrm{sign}((w_0)_j)rt/\lambda_j$. That is weak components of $w_0$ are suppressed, while strong components are boosted.

 The result is Algorithm~\ref{alg:st}, a simple method for computing \textbf{Stage 2} of our proposed two-stage estimator (Algorithm \ref{alg:twostage}). Refer to Figure \ref{fig:algos} for an empirical illustration.

\begin{algorithm}
\caption{Non-uniform soft-thresholding which implements \textbf{Stage 2} of Algorithm \ref{alg:twostage} for diagonal covariance matrix under $\ell_\infty$-attacks.}
\textbf{Inputs: } $\widehat w_0, \widehat \Sigma$
\label{alg:st}
\begin{algorithmic}[1]
    \State \textbf{Compute} $\widehat c=\max_{1 \le j \le d}|(\widehat w_0)_j|\widehat\lambda_j$.
    \State \textbf{For each} $t$ in a finite grid of values between $0$ and $\widehat c$, use held-out to retain the value of $t$ for which the adversarial risk of $\widehat w(t)$ is minimal, where each component of $\widehat w(t)$ is given by
    $\widehat w(t)_j = \mbox{ST}((\widehat w_0)_j;rt/\widehat \lambda_j)$.
    \State \textbf{Return} $\widehat w(t)$.
    \end{algorithmic}
    \end{algorithm}
% Using the well-known \emph{Moreau Decomposition Formula} together with the facts that (1) The convex conjugate of a norm is the (convex-analytic) indicator function of the unit-ball of the dual norm; (2) the proximal operator of the indicator function of a closed convex set is the orthogonal projection operator thereupon,
% we know that, $\mbox{prox}_{\tau \|\cdot\|_\star}(w_0)$ is zero for $\|w_0\| \le \tau$, the above tractable procedure can be extended to any norm!

% \input{sections/minimax}
% \input{sections/experiments}
\section{Concluding Remarks}
In our work, we have undertaken a study of the robustness of gradient-descent (GD), under test-time attacks in the context of linear regression. Our work provides a clear characterization of when GD and a modified version (GD+) --with feature-dependent learning rates-- can succeed to achieve the optimal adversarial risk (up to an absolute constant). This characterization highlights the effect over the covariance structure of the features, and also, of the norm used to measure the strength of the attack.

Finally, our paper proposes a statistically consistent and simple two-stage estimator which achieves optimal adversarial risk in the population regime, up to within a constant factor. Our proposed estimator adapts to attacks w.r.t general norms, and to the covariance structure of the features.

\clearpage
\newpage
\bibliography{litterature}
\clearpage
\newpage
\onecolumn
\appendix

\addtocontents{toc}{\protect\setcounter{tocdepth}{2}}
\begin{center}
    \noindent\rule{17cm}{3pt} \vspace{0.4cm}
    
    \huge \textbf{Supplementary Material} \\ ~\\[-0.5cm]
    \Large \textbf{Robust Linear Regression: Gradient-descent, Early-stopping, and Beyond} 
    
    \noindent\rule{17cm}{1.2pt}
\end{center}

\tableofcontents
\section{Further Details of Experiments}
In this section we detail the experiments presented in the main paper, to empirically verify our theoretical results. In the Figures~\ref{fig:l2-case},~\ref{fig:counter-example-1},~\ref{fig:sufficient} and~\ref{fig:l1-attack-gene}, we consider only $d=2$-dimensional cases as we do not have access to the optimal value of the adversarial risk and propose to compute it using a grid-search. 

\subsection{Case of Optimality of GD under Euclidean Attacks} In Figure~\ref{fig:l2-case}, we plot the optimal value of the adversarial risk under Euclidean attacks obtained along the path of the GD scheme and compare it with the true optimal adversarial risk when varying the attack strength $r$. We observe that we recover our statement proved in Proposition~\ref{prop:GDisOpt}, that is that GD is optimally robust (up to an absolute constant)
and in fact we observe that we have exact equality between these two risks. We also compare their standard risks and observe that they are also equal. It is important to notice that for $r$ sufficiently small, the optimally robust predictor has a null standard risk and therefore is also optimal, which is the result proved in Proposition~\ref{prop:generative}.

\subsection{Suboptimality of GD under General Mahalanobis Attacks} In Figure~\ref{fig:counter-example-1}, we consider the example proposed in Proposition~\ref{thm:GDisBadforMan} and we plot the adversarial risks of the vanilla GD presented in Eq.~(\ref{eq:gdynamics}), as well as the one obtained by our modified GD+ presented in Eq.~(\ref{eq:GD+}) and compare them to the optimal adversarial risk when varying $m$. We show that vanilla GD can be arbitrarily sub-optimal as $m$ goes to infinity while our modified GD+, when selecting the adapted transformation $B(m)$, is able to reach the optimal adversarial risk.

\subsection{Sufficient Condition for Optimality } In Figure~\ref{fig:sufficient}, we consider a simple example where the Condition~\ref{cond:dense} is satisfied and illustrate that in that case the uniform shrinkage strategy, obtained by applying a GD scheme on the normalized data (i.e. by considering GD+ with $M=\Sigma^{-1/2}$), is optimally robust (as shown in Theorem~\ref{thm:GDrocksdeterministic}).
For that purpose we consider a dense generative model $w_0$ under $\ell_{\infty}$-attacks and a covariance matrix $\Sigma$ with a small condition number (in fact we consider $\Sigma=I_d$ for simplicity) and we show that the adversarial risk of the uniform shrinkage strategy matches the optimal one.

\subsection{Case of $\ell_p$-Norm Attacks} 
In Figure~\ref{fig:l1-attack-gene}, we illustrate the result obtained in Theorem~\ref{thm-lp-optimal}. For that purpose we consider a general case where the covariance matrix is diagonal and and sampled according to an Exponential distribution and the generative model is also chosen at random according to a Gaussian distribution. We consider the case of $\ell_{\infty}$-attacks as it is the most used setting in practice. In order to select $M$ in GD+, we propose to solve the following system in $D$ (with $D$ diagonal and positive definite)
\begin{align*}
    (I_d - \exp(-tD))w_0 = w_{opt},
\end{align*}
where $w_{opt}$ is obtained by solving the adversarial problem using a grid-search and $w_0$ is the generative model. Then we consider our GD+ scheme using $M=(D\Sigma^{-1})^{1/2}$. Indeed, in the proof of Theorem~\ref{thm-lp-optimal}, we propose a constructive approach in order to show the optimality of GD+ which simply requires to find a diagonal matrix $D$ solving the above system. We compare the adversarial risk of our GD+ with this specific choice of $M$ against the true optimal adversarial risk and we observe that the two curves coincide when varying the radius $r$.

\subsection{Experiments for Algorithm \ref{alg:twostage}, the Two-Stage Estimator Proposed in Section \ref{sec:algorithmic}}
In Figure \ref{fig:algos}, we plot the adversarial risk computed by Algorithm \ref{alg:twostage}, in the case of $\ell_\infty$-norm attacks. the experimental setting here is: input-dimension $d=200$; covariance matrix $\Sigma$ of the features is heterogeneous diagonal, with diagonal entries drawn iid from an exponential distribution with mean $\lambda=1$; the generative $w_0$ is $s$-sparse vector with $s=10$, normalized so that $\|w_0\|_\Sigma^2 = 1$.
As explained in the figure legends, the acronym "1+3" means it is computed via the Algorithm \ref{alg:st} (which works for diagonal covariance matrix and $\ell_\infty$-norm attack), while "1+2" means \textbf{Stage 2} of our Algorithm \ref{alg:twostage} is computed via the primal-dual algorithm (Algorithm \ref{alg:pd}, which works for every feature covariance matrix $\Sigma$). 

    Notice how the adversarial risk of the estimator improves with number of samples $n$, as expected (Theorem \ref{thm:stats}). Also, unlike "1+3", for small values of the attack strength $r$ the adversarial risk achieved by "1+2"  is dominated by optimization error incurred in \textbf{Stage 2} (Algorithm \ref{alg:pd}). Thanks to Proposition \ref{prop:pd}, this can be alleviated by running more iterations of Algorithm \ref{alg:pd}.

In Figure \ref{fig:comparison}, we also compare our proposed algorithm with the Two-Stage Estimator Proposed in \cite{Xing2021}. Here,  the attack norm is Euclidean, and the covariance matrix of the features $\Sigma$ is as in the above paragraph.

\section{Preliminaries}
\subsection{Additional Notations}
\label{subsect:additional_notations}
The asymptotic notation $F(d) = O(G(d))$ (also written $F(d) \lesssim G(d)$) means there exists a constant $c$ such that $F(d) \le c\cdot G(d)$ for sufficiently large $d$, while $F(d) = \Omega(G(d))$ means $G(d) = O(F(d))$, and $F(d) = \Theta(G(d))$ or $F(d) \asymp G(d)$ means $F(d) \lesssim G(d) \lesssim F(d)$. Finally, $F(d) = o(G(d))$ means $F(d)/G(d) \to 0$ as $d \to \infty$. Thus, $F(d)=O(1)$ means that $F(d)$ is bounded for sufficiently large $d$; $F(d)=\Omega(d)$ means that $F(d)$ is bounded away from zero for sufficiently large $d$; $F(d)=o(1)$ means that $F(d) \to 0$. The acronym w.h.p is used to indicate that a statement is true except on an event of probability $o(1)$.

Also recall all the notations introduced in the beginning of Section \ref{sec:preliminaries}.

\subsection{Proof of Lemma~\ref{lm:fench}: Analytic Formula for Adversarial Risk}
Recall the adversarial risk functional $E$ defined in \eqref{eq:Ewr}.

\fench*
For the proof, we will need the following auxiliary lemma.

\begin{restatable}{lm}{}
For any $x,w \in \mathbb R^d$, $r \ge 0$ and $y \in \mathbb R$, the following identity holds
\begin{eqnarray}
    \sup_{\|\delta\| \le r}((x+\delta)^\top w - y)^2 = (|x^\top w - y| + r\|w\|_\star)^2.
\end{eqnarray}
\label{lm:wellknown}
\end{restatable}
\begin{proof}
Note that $h(x,\delta)/2 = \eta(x)^2/2 + g(x,\delta)/2$, where $g(x,\delta) := w(\delta)^2-2\eta(x)w(\delta)$, and $\eta(x) := w(x) - y$, and $w(x) := x^\top w$. Now, because the function $z \to z^2/2$ is its own Fenchel-Legendre conjugate, we can "dualize" our problem as follows
\begin{eqnarray*}
\begin{split}
\sup_{\|\delta\| \le r}g(x,\delta)/2 &= \sup_{\|\delta\|_\star \le r}-\eta(x)w(\delta) + \sup_{z \in \mathbb R} zw(\delta) - z^2/2\\
&=\sup_{z \in \mathbb R}-z^2/2 + \sup_{\|\delta\| \le r}(z-\eta(x))w(\delta)\\
&= \sup_{z \in \mathbb R}r\|w\|_\star|z-\eta(x)| - z^2/2\\
&= \sup_{s \in \{\pm 1\}}\sup_{z \in \mathbb R}rs(z-\eta(x))-z^2/2\\
&= \sup_{s \in \{\pm 1\}}-r\|w\|_\star s\eta(x)+\sup_{z \in \mathbb R}r\|w\|_\star sz-z^2/2\\
&= \sup_{s \in \{\pm 1\}}-r\|w\|_\star s\eta(x) + r^2\|w\|_\star^2 /2\\
&= r\|w\|_\star |\eta(x)|+r^2\|w\|^2_\star /2.
\end{split}
\end{eqnarray*}
We deduce that
\begin{align*}
\sup_{\|\delta\| \le r}h(x,\delta)/2 &= \eta(x)^2/2+r\|w\|_\star |\eta(x)|+r^2\|w\|_\star^2 /2 = (|\eta(x)|+r\|w\|_\star)^2/2,
\end{align*}
as claimed.
\end{proof}
\begin{proof}[Proof of Lemma \ref{lm:fench}]
Indeed, thanks to Lemma \ref{lm:wellknown} with $y=x^\top w_0$, one has
$$
E^{\|\cdot\|}(w,w_0,r) := \mathbb E_{x \sim N(0,\Sigma)}\sup_{\|\delta\|\le r} h(x,\delta) = \mathbb E_{x \sim N(0,\Sigma)}[(\eta(x)+r\|w\|_\star)^2],
$$
where the functions $h$ and $\eta$ are as in the proof of Lemma \ref{lm:wellknown}. The result then follows upon noting that
\begin{eqnarray*}
\begin{split}
\mathbb E_{x \sim N(0,\Sigma)}[\eta(x)^2] &= \mathbb E_{x \sim N(0,\Sigma)}[(x^\top w - x^\top w_0)^2] = \|w-w_0\|_\Sigma^2,\\
\mathbb E_{x \sim N(0,\Sigma)}|\eta(x)| &= \mathbb E_{x \sim N(0,\Sigma)}|x^\top w - x^\top w_0| = c_0\|w-w_0\|_\Sigma, 
\end{split}
\end{eqnarray*}
where $c_0 := \sqrt{2/\pi}$.
\end{proof}

\subsection{Proof of Lemma~\ref{lm:sandwich}: Proxy for Adversarial Risk}
Recall the definition of the adversarial risk proxy $\widetilde E$ from \eqref{eq:Etilde}.
\sandwich*

\begin{proof}
Let us first show the following useful Lemma.
\begin{restatable}{lm}{}
For any $a,b,c \ge 0$ with $c \le 1$, it holds that
\begin{eqnarray}
(a+b)^2 \geq  a^2 + b^2 + 2abc \geq \frac{1+c}{2}(a+b)^2,
\end{eqnarray}
with equality if $c=1$.
\label{lm:blend}
\end{restatable}
\begin{proof}
Let $h(a,b,c):=a^2+b^2+2abc$.
For the LHS, it suffices to observe that $h(a,b,c) \le h(a,b,1) = (a+b)^2$. For the RHS, WLOG assume that $a \ne 0$, and set $t:=b/a \ge 0$. Observe $\dfrac{h(a,b,c)}{(a+b)^2} = \dfrac{1+ t^2 + 2ct}{(1+t)^2}$, which is minimized when $t=1$, because $0 \le c \le 1$ by assumption. We deduce that $h(a,b,c)/(a+b)^2 \geq (1+1+2c)/(1+1)^2 = (1+c)/2$, as claimed.
\end{proof}
Now, recall from Lemma \ref{lm:fench} that $E^{\|\cdot\|}(w,w_0,r)= \|w-w_0\|_{\Sigma}^2+r^2\|w\|_\star^2+2c_0 r\|w-w_0\|_\Sigma\|w\|_\star$ with $c_0:=\sqrt{\pi/2}$, and by denoting $a=\|w-w_0\|_{\Sigma}$, $b=r\|w\|_\star$ and $c=c_0$ the result follows directly from Lemma \ref{lm:blend}.
\end{proof}

\section{Proof of the Results of Section \ref{sec:gd}: Gradient-descent for Linear Regression}
\subsection{Proof of Proposition~\ref{prop:generative}: Sub-optimality of the Generative Model w.r.t Adversarial Risk}
\generative*
\begin{proof}
In~\cite[Proposition 1]{Xing2021}, the authors already show the optimality of $w_0$ for $r\leq \sqrt{2/\pi}\frac{\Vert w_0\Vert_2}{\Vert w_0\Vert_{\Sigma^{-1}}}$. Now remark that for $w=0$, one has
\begin{align}
\label{lower-bound-opt}
    E^{\Vert\cdot\Vert_2}_{opt}(w_0,r)\leq  E^{\Vert\cdot\Vert_2}(0,w_0,r)\leq \Vert w_0\Vert_{\Sigma}^2
\end{align}
Moreover, we have that for all $r\geq 0$, $E^{\Vert\cdot\Vert_2}(w_0,w_0,r)=r^2\Vert w_0\Vert_2^2$. Combining it with Eq.~(\ref{lower-bound-opt}) gives desired result.
\end{proof}

\subsection{Proof of Proposition~\ref{prop:GDisOpt}: Optimality of Gradient-Descent (GD) in Case of Euclidean Attacks}
Recall the GD dynamics $t \mapsto w(t,w_0)$ given in \eqref{eq:gdynamics}.
We restate the result here for convenience.
\GDisOpt*
For the proof, we will need the following lemma.
\begin{lemma}
For any $\beta\geq 0$, let us denote $w_{\beta}(w_0)=\left(\beta I_d+\Sigma\right)^{-1}\Sigma w_ 0$ the ridge solution. Then for all $r\geq 0$ and $t\geq 0$, we have
\begin{align*}
    E^{\Vert\cdot\Vert_2}(w(t,w_0),w_0,r)\leq E^{\Vert\cdot\Vert_2}(w_{1/t}(w_0),w_0, 1.2985 r).
\end{align*}
\label{gd-ridge}
\end{lemma}
\begin{proof}
In the Euclidean case we obtain a simple expression of $E^{\Vert\cdot\Vert_2}(\cdot,w_0,r)$. Let us denote $\{v_1,\dots,v_d\}$ an orthonormal basis where $\Sigma$ can be diagonalized, $\lambda_1\geq \dots\geq \lambda_d$ the eigenvalues of $\Sigma$ and set $c_j:= w_0^\top v_j$ for all $j\in [d]$. The adversarial risk associated to $w(t)$ can be written as
\begin{align*}
    E^{\Vert\cdot \Vert_2}(w(t,w_0),w_0,r)&= E_1(t,w_0) + r^2 E_2(t,w_0) + 2 r \sqrt{2/\pi} \sqrt{E_1(t,w_0)E_2(t,w_0)},\text{ where}\\
    E_1(t,w_0)&:= \sum_{j=1}^d (1 - \alpha \lambda_j)^{2t} \lambda_j  c_j^2,\,\text{ and }
    E_2(t,w_0):= \sum_{j=1}^d (1 - (1-\alpha\lambda_j)^{t})^2c_j^2\;.
\end{align*}
By using the elementary fact that $1-x\leq \exp(-x)\leq 1/(1+x)$ for all $x \ge 0$, we obtain first that 
\begin{align*}
     E_1(t,w_0)&=\sum_{j=1}^d \exp(-\lambda_j 2t) \lambda_j  c_j^2\leq \sum_{j=1}^d \frac{\lambda_j}{(1+t \lambda_j)^2} c_j^2= \Vert w_{1/t}(w_0) - w_0\Vert_{\Sigma}^2=E_1(w_{1/t }(w_0))\;.
\end{align*}
In addition, using the elementary fact\footnote{For example, see ~\cite[Lemma 7]{continuousGDAnul2019}.} that $1-\exp(-x)\leq  1.2985 x/(1+x)$ for all $x \ge 0$, it follows that
\begin{align*}
     E_2(t,w_0)&= \sum_{j=1}^d(1 - \exp(-\lambda_j t))^2 c_j^2 \leq 1.6862 \sum_{j=1}^d\left(\frac{\lambda_j t}{1+\lambda_j t}\right)^2 c_j^2 =1.2985^2 E_2(w_{1/ t}(w_0))
\end{align*}
and the result follows.
\end{proof}
Finally, we will need the following elementary lemma.
\begin{lemma}
\label{lem-trivial-rescale}
for any $\gamma\geq 1$, $r\geq 0$, we have
     $E^{\Vert\cdot \Vert}(w,w_0,\gamma r)\leq \gamma^2 E^{\Vert\cdot \Vert}(w,w_0,r)$
\end{lemma}
\begin{proof}
Indeed we have
\begin{align*}
    \|w-w_0\|_{\Sigma}^2 &\leq \gamma^2  \|w-w_0\|_{\Sigma}^2~~\text{and}~~2c_0 \gamma r\|w-w_0\|_\Sigma\|w\|_\star\leq  2c_0 \gamma^2 r\|w-w_0\|_\Sigma\|w\|_\star
\end{align*}
from which follows the result.
\end{proof}
We are now in place to prove Proposition \ref{prop:GDisOpt}.
\begin{proof}[Proof of Proposition \ref{prop:GDisOpt}]
First note that in~\cite[Proposition 1]{Xing2021}, the authors show that as soon as $r\leq \sqrt{2/\pi}\frac{\Vert w_0\Vert_2}{\Vert w_0\Vert_{\Sigma^{-1}}}$, the robust predictor minimizing Eq.~(\ref{eq:Ewrdef}) is $w_0$ which is attained by the GD dynamic when $t$ goes to infinity. They also show that when $r\geq \sqrt{\pi/2}\frac{\Vert w_0\Vert_{\Sigma^2}}{\Vert w_0\Vert_{\Sigma}}$, the optimal solution is $0$ which is reached by the vanilla GD scheme for $t=0$. Let us now show the result for $\sqrt{2/\pi}\frac{\Vert w_0\Vert_2}{\Vert w_0\Vert_{\Sigma^{-1}}}<r<\sqrt{\pi/2}\frac{\Vert w_0\Vert_{\Sigma^2}}{\Vert w_0\Vert_{\Sigma}}$. Indeed, in~\cite[Proposition 1]{Xing2021}, they show that for any $r$ in this range, there exists $\beta^*> 0$, such that 
\begin{align*}
    E^{\Vert\cdot\Vert_2}_{opt}(w_0,r)=E^{\Vert\cdot\Vert_2}(w_{\beta^*}(w_0),w_0,r)
\end{align*}
Then by considering $t_{opt}:=1/\beta^*$, we obtain that
\begin{align*}
    E^{\Vert\cdot\Vert_2}(w(t_{opt},w_0),w_0,r)\leq  E^{\Vert\cdot\Vert_2}(w_{1/t_{opt}}(w_0),w_0,1.2985r)
    \leq 1.2985^2 E^{\Vert\cdot\Vert_2}_{opt}(w_0,r)
\end{align*}
where the first inequality follows from Lemma~\ref{gd-ridge} and the second from Lemma~\ref{lem-trivial-rescale}, which gives the desired result. 
 \end{proof}
 
 \subsection{Uniform Shrinkage of $w_0$ is Sub-optimal in General}
In the case where the isotropic case where $\Sigma=I_d$ and the attacker's norm is Euclidean, we can show that GD is optimally robust and so for any radius $r$. In that case, the GD scheme is simply a uniform shrinkage of the generative model $w_0$ as we have that
\begin{align*}
    w(t,w_0)=(1-\exp(-t))w_0.
\end{align*}
Let us now recall our statement for convenience.
\isoexact*
\begin{proof}
First recall, from Proposition~\ref{prop:GDisOpt} that for $r\leq \sqrt{2/\pi}$ or $r\geq \sqrt{\pi/2}$, we have already shown the result. Now observe that in the isotropic case, the ridge solution is in fact also a uniform shrinkage of $w_0$, as we have $w_{\beta}(w_0)=\left(\beta I_d+\Sigma\right)^{-1}\Sigma w_ 0 = \frac{1}{\beta + 1} w_0$, and therefore the optimality of the GD follows directly from~\cite[Proposition 1]{Xing2021}.
\end{proof}

% it is easy to show that a shrinkage strategy $w=\alpha w_0$ achieves optimal adversarial risk up to within an absolute multiplicative constant. That is, 
% the following result (not included in the main paper) shows that in the case of anisotropic features, even though GD achieves optimal adversarial risk (up to wiin a multiplicative factor), i.e
%  \begin{eqnarray}
% \sup_{r \ge 0}\frac{E_{shrink}^{\|\cdot\|_2}(w_0,r)}{E_{opt}^{\|\cdot\|_2}(w_0,r)} \le \alpha \approx 1.1124,
% \end{eqnarray}
% where $E_{shrink}^{\|\cdot\|}(w_0,r) := \inf_{\alpha \in [0,1]}E^{\|\cdot\|}(\alpha w_0,w_0,r)$. 
% Incidentally, in this isotropic regime, the GD dynamics is a shrinkage of $w_0$, namely $w(t,w_0) = \alpha(t)w_0$, with $\alpha(t) = 1-e^{-t}$, and thus $E_{shrink}^{\|\cdot\|_2}(w_0,r)$ is achieved by GD. 
The next result shows that in the anisotropic case where $\Sigma \ne I_d$, uniform shrinkage can achieve arbitrarily bad adversarial risk, while GD continues to be optimal (up to within an absolute multiplicative constant), thanks to Proposition \ref{prop:GDisOpt}.
\begin{restatable}{prop}{}
For sufficiently large values of the input-dimension $d$, it is possible to construct a  feature covariance matrix $\Sigma \in \mathcal S^{++}_d(\mathbb R)$ and   generative model $w_0 \in \mathbb R^d$ such that $\trace(\Sigma),\|w_0\|_{\Sigma} = \Theta(1)$ and
\begin{align}
E_{shrink}^{\|\cdot\|_2}(w_0,1/\sqrt d)=\Omega(1)\label{eq:crash1},\\
E_{opt}^{\|\cdot\|_2}(w_0,1/\sqrt d) = o(1)\label{eq:crash2}.
\end{align}

\end{restatable}

\begin{proof}
Fix $\beta \in (1,\infty)$, and any integer $k \ge 1$, set $\lambda_k := k^{-\beta}$. For a large positive integer $d$, consider the choice of the feature covariance matrix $\Sigma = \diag(\lambda_1,\ldots,\lambda_d)$. Thus, the eigenvalues $\lambda_k$ of $\Sigma$ decay polynomially as a function of their rank $k$, and $\Sigma$ is very badly conditioned (its condition number is $d^\beta$, a nontrivial polynomial in the input-dimension $d$). Consider a random choice of the generative model $w_0 \sim N(0,I_d)$. We will prove the following stronger statement
\begin{mdframed}
\textbf{Claim.} \emph{For  large values of the input-dimension $d$ and the above choice of feature covariance matrix $\Sigma$, the estimates \eqref{eq:crash1} and \eqref{eq:crash2} hold  w.h.p over the choice of $w_0 \sim N(0,I_d)$.}
\end{mdframed}
First observe that, by  elementary concentration it holds w.p $1-e^{-Cd}$ over the choice of $w_0 \sim N(0,I_d)$ that  \begin{align*}
\|w_0\|_2^2 &\asymp \mathbb E\|w_0\|_2^2 = d,\\
\|w_0\|_\Sigma^2 &\asymp \mathbb E\|w_0\|_\Sigma^2 = \trace(\Sigma)  = \sum_{k=1}^dk^{-\beta} =\Theta(1).
\end{align*}
Since $E_{shrink}^{\|\cdot\|_2}(w_0,r)=\|w_0\|_\Sigma^2 \land r^2\|w_0\|_2^2$, we deduce that for any $r \ge 0$, the following holds w.p $1-e^{-Cd}$. 
\begin{eqnarray}
E_{shrink}^{\|\cdot\|_2}(w_0,r)  \asymp 1 \land r^2 d.
\label{eq:shrek}
\end{eqnarray}
On the other hand, thanks to Lemma \ref{lm:sandwich} one has \begin{eqnarray}
\sqrt{E_{opt}^{\|\cdot\|_2}(w_0,r)} \asymp \inf_{w\in \mathbb R^d}\|w-w_0\|_\Sigma + r\|w\|_2 = \sup_{z \in S}z^\top w_0,
\label{eq:shreklet}
\end{eqnarray}
where $S := B_{\|\cdot\|_{\Sigma^{-1}}}^d \cap rB_2^d = \{z \in \mathbb R^d \mid \|z\|_2 \le r,\,\|z\|_{\Sigma^{-1}} \le 1\}$. In particular, we deduce that $\mathbb E \sqrt{E_{opt}^{\|\cdot\|_2}(w_0,r)} \asymp \omega(S)$, i.e the \emph{Gaussian width} of $S$, defined by
\begin{eqnarray}
\omega(S) = \mathbb E\left[\sup_{z \in S}z^\top w_0\right].
\end{eqnarray}
Observe that, the function $G:a \mapsto \sup_{z\in S}z^\top a$ is $r$-Lipschitz w.r.t to the Euclidean norm. Indeed, for any $a,b \in \mathbb R^d$,
$$
|G(a)-G(b)| = \left|\sup_{z \in S}z^\top a - \sup_{z \in S}z^\top b\right| \le \sup_{z \in S}|z^\top (a-b)| \le \sup_{z \in rB_2^d}|z^\top (a-b)| = r\|a-b\|_2.
$$
Therefore, by concentration of Lipschitz functions of Gaussians, we have for any positive $t$,
\begin{eqnarray}
\label{eq:lipschitzconcentration}
\left|\sup_{z \in S}z^\top w_0 - \omega(S)\right| \le t, \text{ w.p }1-2e^{-t^2d /(2r^2)}.
\end{eqnarray}
Subsequently, we will take $t$ sufficiently small.

We now upper-bound $\omega(S)$. Referring to
  ~\cite[Example 4]{Gauss2Kolmogorov}\footnote{Alternatively, it is easy to show that $\omega(S) \le \sum_{k=1}^d \min(\lambda_k,r^2)$, and then observe that }, one observes that
\begin{eqnarray}
\label{eq:wSbound}
\omega(S) \asymp r^{1-1/\beta}.
\end{eqnarray}
  Combining  \eqref{eq:shrek}, \eqref{eq:wSbound}, and \eqref{eq:lipschitzconcentration} applied with $t=\omega(S) \asymp r^{1-1/\beta}$, we obtain that if $r \lesssim 1/\sqrt{d}$, then it holds w.p $1-e^{-Cd}-2e^{-d/(2r^{2/\beta})}$ over the choice of $w_0 \sim N(0,I_d)$ that
\begin{align}
E_{shrink}^{\|\cdot\|_2}(w_0,r) & \asymp 1 \land r^2 d = r^2d,\\
E_{opt}^{\|\cdot\|_2}(w_0,r) &\asymp r^{2(1-1/\beta)}. 
\end{align}
% Moreover, if $1/\sqrt{d^{\beta}} \ll r \lesssim 1/ \sqrt d$, then $r^{2(1-1/\beta)}=o(1)$.
Plugging $r=1/\sqrt d$ in the above then completes the proof of the claim.
\end{proof}
 
\subsection{Proof of Proposition~\ref{thm:GDisBadforMan}: Sub-optimality of GD in Case of Mahalanobis Attacks} 
\GDisBadforMan*
\begin{proof}
Let us first compute a lower bound on the optimal adversarial risk along the path of the vanilla GD. Indeed, first note that from Lemma~\ref{lm:sandwich}, one obtains that
\begin{align*}
    \inf_{t\geq 0}\alpha E^{\Vert\cdot\Vert_{B}}(w(t,w_{0}),w_{0},r) \geq  \inf_{t\geq 0}\widetilde E^{\Vert\cdot\Vert_{B}}(w(t,w_{0}),w_{0},r).
\end{align*}
In the particular case where $\Sigma=I_d$, we have that $w(t,w_{0}):=(1-\exp(-t))w_0$ and therefore we obtain that 
\begin{align*}
    \inf_{t\geq 0}\widetilde E^{\Vert\cdot\Vert_{B}}(w(t,w_{0}),w_{0},r)&=\min(\Vert w_0\Vert_2^2, r^2\Vert w_0\Vert_{B^{-1}}) =\min(1/m + 1, r^2(1+1/m))
\end{align*}
Therefore, one has
\begin{align*}
     \inf_{t\geq 0}E^{\Vert\cdot\Vert_{B}}(w(t,w_{0}),w_{0},r) \geq \alpha \min(1/m + 1, r^2(1+1/m))\underset{m\rightarrow \infty}{\longrightarrow} 1/\alpha \min(1,r^2)
\end{align*}
Let us now compute the value of the adversarial risk proxy $\widetilde E(w,w_0,r)$. To this end, remark that in our specific example, one can write the proxy adversarial risk at $w=(w_1,w_2)$ as $\widetilde{E}^{\Vert\cdot\Vert_B}(w,w_0,r) = (\alpha + r\beta)^2$, where
\begin{align}
\alpha &= \alpha(w) = \sqrt{(w_1-1/\sqrt{m})^2+(w_2-1)^2 },\\
\beta &= \beta(w) = \sqrt{m w_1^2 + (1/m) w_2^2},
\end{align}
The critical points of $\widetilde{E}^{\Vert\cdot\Vert_B}(\cdot,w_0,r)$ are either $0$ ($\beta=0$), $w_0$ ($\alpha=0$), or must satisfy $\nabla_w(\alpha + r \beta) = 0$, i.e
$\dfrac{(w-w_0)}{\alpha} + \dfrac{rB^{-1} w}{\beta} = 0$. Let us consider the latter case and let us define $t :=\beta/\alpha$ satisfying the last equation. Then it is equivalent to write $t(w-w_0) + rB^{-1}w = 0$. We conclude that the minimum of $\widetilde{E}^{\Vert\cdot\Vert_B}(\cdot,w_0,r)$ is of the form
\begin{eqnarray}
w = w(t) := D(t) w_0,
\end{eqnarray}
where $D(t) = t (t I_d + rB^{-1})^{-1}  = I_d-(t I_d+rB^{-1})^{-1}rB^{-1}$ is the $d \times d$ diagonal matrix with diagonal entries $(t / (t + rm),t / (t + r/m))$. In that case we obtain that
\begin{align*}
    \widetilde{E}^{\Vert\cdot\Vert_B}_{opt}(w_0,r) &=\alpha(w(t))^2(1+rt)^2\\
    &=r^2(1+rt)^2\left(\frac{m}{(t+rm)^2}+\frac{1}{m^2(t+r/m)^2)}\right)
\end{align*}
Under the assumption that: for $m$ sufficiently large, such $t=t(m)$ exists and the sequence $(t(m))_{m\geq 0}$ converges towards an absolute and positive constant $c$, we obtain that
\begin{align*}
     \widetilde{E}^{\Vert\cdot\Vert_B}_{opt}(w_0,r)\underset{m\rightarrow \infty}{\longrightarrow}0
\end{align*}
Let us now show that such $t(m)$ exists and that it does converges towards an absolute constant. Indeed, by definition, if $t$ exists, then it must satisfies $\|w(t)\|_B = t\|w(t)-w_0\|_\Sigma$, which gives
\begin{eqnarray}
\frac{1}{(t+rm)^2}+ \frac{m}{(t+r/m)^2} = 
r^2\left[\frac{m}{(t+rm)^2}+\frac{m^2}{(t+r/m)^2} \right]
\end{eqnarray}
Therefore as soon as $r^2/m < 1$ and $r^2m > 1$, we obtain that such $t$ exists and satisfies
\begin{align*}
t = \frac{r/m \sqrt{(r^2m-1)} - rm\sqrt{(1/m)(1-r^2/m)}}{\sqrt{(1/m)(1-r^2/m)} - \sqrt{(r^2m-1)}}\underset{m\rightarrow \infty}{\longrightarrow}1
\end{align*}
and the result follows.
\end{proof}

\section{Proofs of Results of Section \ref{sec:gdplus}: An Adapted Gradient-Descent (GD+)}
\subsection{Proof of Proposition~\ref{prop:GDPlusIsGood}: Optimality of GD+ in Case of Mahalanobis Norm Attacks}
\GDPlusIsGood*

\begin{proof}
According to Proposition~\ref{prop:GDisOpt}, we already know that early-stopped GD scheme on data with covariance $\Sigma$ generated by a model $w_0$ achieves adversarial risk which is optimal for Euclidean / $\ell_2$ attacks, up to within a multiplicative factor of $\beta$. Now remark that by considering our transformation on the data $\tilde{x}:=Mx$, we are defining a GD scheme on data with covariance $M\Sigma M^{\top}$ generated by $(M^{\top})^{-1}w_0$. We deduce that, early-stopped GD+ achieves the optimal value (up to withing a multiplicative factor of $\beta$) in optimization problem
\begin{align*}
    \min_{w\in\mathbb{R}^d} \Vert w - (M^{\top})^{-1}w_0\Vert_{M\Sigma M^{\top}}^2 + r^2\Vert w\Vert_2^2 + 2rc_0\Vert w - (M^{\top})^{-1}w_0\Vert_{M\Sigma M^{\top}}\Vert w\Vert_2\; .
\end{align*}
Now observe that
\begin{align*}
   \Vert w - (M^{\top})^{-1}w_0\Vert_{M\Sigma M^{\top}}^2 =   \Vert M^{\top}w - w_0\Vert_{\Sigma}^2 
\end{align*}
and therefore, by applying again a change of variables $\tilde{w}=M^{\top}w$, we obtain that,  along the path drawn by  $(M^{\top}w^{M}(t,w_0))_{t\geq 0}$, there exists an optimal solution (up to within a multiplicative factor of $\beta$) for 
\begin{align*}
   \min_{w\in\mathbb{R}^d} \Vert w -w_0\Vert_{\Sigma}^2 + r^2\Vert (M^{T})^{-1}w\Vert_2^2 + 2rc_0\Vert w - w_0\Vert_{\Sigma}\Vert  (M^{T})^{-1}w\Vert_2\; .
\end{align*}
Finally, taking finally $M=B^{1/2}$ gives the desired result, since the dual norm of $\Vert\cdot\Vert_B$ is $\Vert\cdot\Vert_{B^{-1}}$.
\end{proof}

\subsection{Proof of Proposition~\ref{prop:gene-upper}: Case of Optimality of GD+ in Case of General Norm Attacks}
\geneupper*

\begin{proof}
% Let us define $V_{\Vert (M^{T})^{-1}\cdot\Vert,r,w_0,\Sigma}(w)$ as:
% \begin{align*}
%     V_{\Vert (M^{T})^{-1}\cdot\Vert,r,w_0,\Sigma}(w):=\Vert w -w_0\Vert_{\Sigma}^2 + r^2\Vert (M^{T})^{-1}w\Vert_2^2 + 2rc_0\Vert w - w_0\Vert_{\Sigma}\Vert  (M^{T})^{-1}\Vert_2
% \end{align*}
Recall the definition of $ \lambda_{min}^{\Vert\cdot\Vert}$ and $ \lambda_{max}^{\Vert\cdot\Vert}$ from the beginning of Section \ref{subsec:general_attacks} of the main paper. Observe that for all $w\in\mathbb{R}^d$
\begin{align}
\label{eq-lam-up-low}
   \lambda_{min}^{\Vert\cdot\Vert}((M^{T})^{-1}) \Vert w\Vert_{*}\leq \Vert (M^{T})^{-1}w\Vert_2 \leq  \lambda_{max}^{\Vert\cdot\Vert}((M^{T})^{-1}) \Vert w\Vert_{*}\; .
\end{align}
In the following we also define $\Vert\cdot\Vert^M =\Vert M\cdot\Vert_2$. Then, one computes
\begin{align*}
    \inf_t  E^{\Vert\cdot\Vert}(M w^M(t,w_0),w_0,r)
    &\leq  \inf_t E^{\Vert\cdot\Vert^{M}}(M w^M(t,w_0),w_0,r/\lambda_{min}^{\Vert\cdot\Vert}((M^{T})^{-1}))(M w^M(t,w_0))\\
    &\leq \beta \inf_{w\in\mathbb{R}^d} E^{\Vert\cdot\Vert^{M}}(w,w_0,r/\lambda_{min}^{\Vert\cdot\Vert}((M^{T})^{-1}))\\
    &\leq \beta \inf_{w\in\mathbb{R}^d}  E^{\Vert\cdot\Vert}(w,w_0,r\kappa^{\Vert\cdot\Vert}((M^{T})^{-1}))\\
    &\leq \beta\kappa^{\Vert\cdot\Vert}(((M^{T})^{-1}))^2 \inf_w  E^{\Vert\cdot\Vert}(w,w_0,r)
\end{align*}
where the first inequality follows from the LHS of Eq.~(\ref{eq-lam-up-low}), the second from Proposition~\ref{prop:GDPlusIsGood}, the third from RHS of Eq.~(\ref{eq-lam-up-low}), and the last one from Lemma~\ref{lem-trivial-rescale}.

% \begin{align*}
%     \inf_t  V_{\Vert \cdot\Vert_{*},r,w_0,\Sigma}(M w^M(t,w_0))
%     &\leq  \inf_t V_{\Vert (M^{T})^{-1}\cdot\Vert,r/\lambda_{min}^{\Vert\cdot\Vert}((M^{T})^{-1}),w_0,\Sigma}(M w^M(t,w_0))\\
%     &\leq \beta  \inf_w V_{\Vert (M^{T})^{-1}\cdot\Vert,r/\lambda_{min}^{\Vert\cdot\Vert}((M^{T})^{-1}),w_0,\Sigma}(w)\\
%     &\leq \beta \inf_w  V_{\Vert \cdot\Vert_{*},r\kappa^{\Vert\cdot\Vert}(((M^{T})^{-1})),w_0,\Sigma}(w)\\
%     &\leq \beta\kappa^{\Vert\cdot\Vert}(((M^{T})^{-1}))^2 \inf_w  V_{\Vert \cdot\Vert_{*},r,w_0,\Sigma}(w)
% \end{align*}
\end{proof}
\subsection{Proof of Theorem~\ref{thm:GDrocksdeterministic}: A Case of Optimality of GD for Non-Euclidean / General Norm Attacks}
\GDrocksdeterministic*

\begin{proof}
Let  $g(w_0) \in \mathbb R^d$ be a subgradient of $\Vert\cdot\Vert_\star$ at $w_0$. First note that $\|g(w_0)\| \le 1$ and $g(w_0)^\top w_0 =\|w_0\|_\star$ and let us denote
\begin{eqnarray}
\lambda:=\frac{\|g(w_0)\|}{\|g(w_0)\|_{\Sigma^{-1}}}
\end{eqnarray}
Let us now show the following useful Proposition.
\begin{prop}
\label{prop-lower-bound-adv-gene}
Let $w_0\in\mathbb{R}^d$ and $r>0$. Then we have that
\begin{align*}
    \sqrt{\alpha E^{\|\cdot\|}_{opt}(w_0,r)}\geq \sup_{z\in B_2^d \cap rB^d_{\|\Sigma^{1/2}\cdot\|}}\langle z,\Sigma^{1/2}w_0\rangle
\end{align*}
\end{prop}
\begin{proof}
Let us define the support function of a subset $C$ of $\mathbb R^d$ as the function $\gamma_C:\mathbb R^d \to (-\infty,\infty]$, defined by $\gamma_C(x) := \sup_{w \in \mathbb R^d}x^\top w$ and let us introduce the following useful lemma.
\begin{restatable}{lm}{}
Let $A$ and $B$ be two subsets of topological vector space $V$, such that the intersection of the relative interiors of $A$ and $B$ is nonempty. Then, $\gamma_{A \cap B} = \gamma_A \star \gamma_B$,
where $f \star g:V \to (-\infty,+\infty]$ is infimal convolution of two proper convex lower-semicontinuous functions $f,g:V \to (-\infty,+\infty]$, defined by $(f \star g)(v) := \inf_{v' \in V}f(v-v') + g(v')$.
\label{lm:intsupportfunc}
\end{restatable}
Now using Lemma~\ref{lm:sandwich}, recall that
\begin{align*}
    \sqrt{\alpha E^{\Vert\cdot\vert}_{opt}(w_0,r)}\geq   \sqrt{\widetilde{E}^{\Vert\cdot\vert}_{opt}(w_0,r)} &= \inf_w \|w-w_0\|_\Sigma + r \|w\|_\star \\
    &=(\gamma_{B_{\Sigma^{-1}}^d} \star \gamma_{B^d_{\|\cdot\|}})(w_0) \\
    &= \gamma_{S_r}(w_0)~~\text{where}~~S_r := B_{\Sigma^{-1}}^d(1) \cap B_{\|\cdot\|}^d(r)\\
    &=\gamma_{\widetilde{S_r}}(\Sigma^{1/2} w_0)~~\text{where}~~\widetilde{S_r} := B_2^d(1) \cap B_{\|\Sigma^{1/2}\cdot\|}^d(r)
\end{align*}
from which the result follows.
\end{proof}
Define now $z(w_0) := (r \land \lambda)\Sigma^{-1/2}g(w_0) \in \mathbb R^d$, and observe that
\begin{align*}
\|\Sigma^{1/2}z(w_0)\| &\le r\|g(w_0)\| \le r,\\ \|z(w_0)\|_2 &\le \lambda\|g(w_0)\|_{\Sigma^{-1}} \le  \|g(w_0)\|\le 1.
\end{align*}
Thus, $z(w_0) \in B_2^d \cap rB^d_{\|\Sigma^{1/2}\cdot\|}$. Then, thanks to Proposition~\ref{prop-lower-bound-adv-gene}, it follows that
\begin{eqnarray}
\begin{split}
\sqrt{\alpha E_{opt}^{\Vert\cdot\Vert}(r,w_0)} &\ge \max_{z \in B_2^d \cap rB_{\|\cdot\|^{\Sigma}}}z^\top \Sigma^{1/2}w_0\\
&\ge z(w_0)^\top\Sigma^{1/2} w_0= (r \land \lambda)g(w_0)^\top w_0\\
&= (r\land \lambda)\|w_0\|_\star,
\end{split}
\end{eqnarray}
where $\alpha :=1/(1+\sqrt{2/\pi}) \approx 1.1124$.
% Therefore, if the ground-truth model $w_0$ is "dense" in the sense that the following \emph{local condition} holds: (see examples later)
% \begin{restatable}[Sufficient condition]{cond}{}
% \label{cond:dense}
% There exists a subgradient $g(w_0)$ of $\Vert\cdot\Vert_{*}$ at $w_0$ such that $\frac{\|g(w_0)\|}{\|g(w_0)\|_{\Sigma^{-1}}} \gtrsim c\frac{\|w_0\|_{\Sigma}}{ \|w_0\|_\star }$, for some absolute positive constant $c$.
% \end{restatable}
Let us now define
\begin{align*}
    c:=\frac{\lambda \Vert w_0\Vert_*}{ \Vert w_0\Vert_{\Sigma}}
\end{align*}
then $\alpha E_{opt}^{\Vert\cdot\Vert}(w_0,r) \ge c^2\|w_0\|_{\Sigma}^2 \land r^2\|w_0\|_\star^2 \ge (1 \land c^2) (\|w_0\|_{\Sigma}^2 \land r^2\|w_0\|_\star^2)$ which matches
$\inf_{\gamma\in[0,1]}\widetilde{E}^{\Vert\cdot\Vert}(\gamma w_0,w_0,r)$ attained by the optimal uniform rescaling of $w_0$ and the result follows.
\end{proof}

\subsection{Proof of Proposition~\ref{thm:bangbang}: A Case of Sub-optimality of GD to $\ell_\infty$-Norm Attacks}
\bangbang*

\begin{proof}
Let $a,\lambda>0$, $w_0:=(a+\lambda,\dots,a+\lambda,\lambda,\dots,\lambda)$ and let $\alpha\geq 1$ be the number of coordinate such that $(w_{0})_i=a+\lambda$. Concrete values for $a$, $\lambda$, and $\alpha$ will be prescribed later. For such model we obtain that
\begin{align*}
    \Vert w_0\Vert_1 &= d\lambda + \alpha a~~\text{and}\\
     \Vert w_0\Vert_2 &=\sqrt{\alpha(a+\lambda)^2 + (d-\alpha) \lambda^2}\\
     &= \sqrt{\alpha a^2 + d \lambda^2 + 2\alpha a \lambda}.
\end{align*}
Therefore the shrinkage model obtains an error equals to:
\begin{align*}
    E_{shrink}^{\|\cdot\|_\infty}(w_0,r)=\min(\alpha a^2 + d \lambda^2 + 2\alpha a \lambda, r^2( d\lambda + \alpha a)^2).
\end{align*}
Now let $w:=(\underbrace{a,\dots,a}_{\alpha\text{ times}},0,\dots,0)$, we obtain that
\begin{align*}
    \sqrt{E_{opt}^{\|\cdot\|_\infty}(w_0,r)}\leq \Vert w - w_0\Vert_2 + r\Vert w\Vert_1=\sqrt{d}\lambda + r\alpha a
\end{align*}
Let us now consider the case where \begin{eqnarray*}
r=\lambda=\frac{1}{d^{1/4}\log d},\, 
\alpha = d^{1/2},\, a=1.
\end{eqnarray*}
Then, for large $d$ we obtain that
\begin{align*}
&\Vert w - w_0\Vert_2 + r\Vert w\Vert_1\asymp \frac{d^{1/4}}{\log d}\\
  &r\Vert w_0\Vert_1 \asymp  \frac{d^{1/2}}{(\log d)^2},\,\Vert w_0\Vert_2 \asymp d^{1/4}
\end{align*}
from which follows that
\begin{align*}
    \frac{E_{shrink}^{\|\cdot\|_\infty}(w_0,r)}{E_{opt}^{\|\cdot\|_\infty}(w_0,r)}\geq \frac{\min(d^{1/2},\frac{d}{(\log d)^4})}{\frac{d^{1/2}}{(\log d)^2}}\asymp (\log d)^2\rightarrow +\infty,
\end{align*}
as claimed.
\end{proof}

\subsection{Proof of Theorem~\ref{thm-lp-optimal}: Case of Optimality of GD+ to $\ell_p$-Norm Attacks}
\lpoptimal*

For the proof, we will need the following useful lemma.
\begin{restatable}{lm}{}
Given an extended-value convex function (EVCF) $g:\mathbb R^d \to (-\infty,\infty]$, every minimizer of the EVCF $g_a:\mathbb R^d \to (-\infty,\infty]$ defined by $g_a(x) := \|x-a\|_\Sigma + g(x)$ can be written as $x=\mbox{prox}_{\Sigma,tg}(a)$, for some $t \in [0,\infty]$.
\label{lm:subdifferential} 
\end{restatable}
In the above Lemma, we consider the proximal operator
\begin{align}
\label{prox-gene}
    \mbox{prox}_{\Sigma,tg}(a):= \arg\min_{x\in\mathbb{R}^d}(1/2)\Vert x - a\Vert_{\Sigma}^2 + tg(x),
\end{align}
defined w.r.t to the Bregman divergence on $\mathbb R^d$ given by $D(z,z') := (1/2)\Vert z - z'\Vert_{\Sigma}^2$.
\begin{proof}
Let $x \in \mathbb R^d$ be a minimizer of $g_a$. If $x \in \{0,a\}$, then we are done. Otherwise, set $t = t(x):=\|x - a\|_{\Sigma}$. Note that  by first-order optimality conditions, $x$ is a minimizer of $g_a$ iff $0 \in \partial g_a(x)$ iff $0 \in \Sigma(x-a)/t + \partial g(x)$ iff $0 \in (x-a) +  t\Sigma^{-1}\partial g(x)$ iff $x = (I+t\Sigma^{-1} \partial g)^{-1} (a) = \mbox{prox}_{\Sigma,t g}(a)$, by definition of the proximal operator. Indeed recall tht $\partial g$ is a maximal monotone operator (a result of Rockafellar from the 70's \cite{rtmaxmon}), so that $(I+t \Sigma^{-1} \partial g)^{-1}(a)$ is well-defined and single-value.
\end{proof}

% In the next Lemma, we provide a useful properties of $\mbox{prox}_{\Sigma,\Vert\cdot\Vert_q}(a)$ when $\Sigma$ is a diagonal matrix. \textcolor{red}{Elvis: $p$ doesn't appear in the lemma...}
% In the next useful Lemma, we show that as soon as $\Sigma$ is a diagonal matrix, $\mbox{prox}_{\Sigma,\Vert\cdot\Vert_q}(a)$ is in fact a non-uniform shrinkage of $a$.
Let us now introduce another important Lemma in order to show the result.
\begin{restatable}{lm}{}
For any $q\in[1,2]$, $a\in\mathbb{R}^d$, and $\Sigma$ positive definite diagonal matrix, we have
\begin{align*}
    &|\mathrm{prox}_{\Sigma,\Vert\cdot\Vert_q}(a)|\leq |a|~~\text{and}\\
    &\mathrm{sign}(\mathrm{prox}_{\Sigma,\Vert\cdot\Vert_q}(a))= \mathrm{sign}(a)
\end{align*}
where the $|\cdot|$ and $\mathrm{sign}$ are the coordinate-wise operators.
\label{lm:order-prox}
\end{restatable}
% \begin{restatable}{lm}{}
% For any $q\in[1,2]$ and $a\in\mathbb{R}^d$, we have
% \begin{align*}
%     &|\mbox{prox}_{\Sigma,\Vert\cdot\Vert_q}(a)|\leq |a|~~\text{and}\\
%     &\mathrm{sign}(\mbox{prox}_{\Sigma,\Vert\cdot\Vert_q}(a))= \mathrm{sign}(a)
% \end{align*}
% where the $|\cdot|$ and $\mathrm{sign}$ are the coordinate-wise operators.
% % \label{lm:order-prox}
% \end{restatable}
\begin{proof}
% First recall that we have for any norm $g(\cdot)=\Vert\cdot\Vert$ 
% \begin{align*}
%     \mbox{prox}_{g}(a)=(I_d - \mbox{proj}_{B_{g_{*}}}(\cdot))(a)
% \end{align*}
% where we have considered the orthogonal projection and $g_*$ the dual norm of $g$.
A simple computation reveals that
\begin{align*}
    \mbox{prox}_{\Sigma,g} = \Sigma^{-1/2} \circ \mbox{prox}_{g \circ \Sigma^{-1/2}} \circ \Sigma^{1/2},
\end{align*}
for any EVCF $g$. Using the Moreau decomposition formula, we can further express
\begin{align*}
    \mbox{prox}_{\Sigma,g}(a) &=  \Sigma^{-1/2}[\Sigma^{1/2} a - \mbox{prox}_{(g \circ \Sigma^{-1/2})^\star}(\Sigma^{1/2} a)]\\
&=  a - \Sigma^{-1/2}\mbox{prox}_{g^\star \circ \Sigma^{1/2}}(\Sigma^{1/2} a),
\end{align*}
where $g^\star$ is the convex conjugate of $g$, and we have used the fact that $(g \circ \Sigma^{-1/2})^\star = g^\star \circ \Sigma^{1/2}$.
Now, taking $g=\|\cdot\|_q$, one notes that $g^\star$ is the indicator function of the unit-ball $B_{\|\cdot\|_p}^d$ for the norm $\|\cdot\|_p$, where $p=p(q) \in [2,\infty]$ is the harmonic conjugate of $q$. Therefore, we obtain that 
\begin{align*}
     \mbox{prox}_{\Sigma,g}(a)  = a - \Sigma^{-1/2}\mbox{proj}_{B^d_{\Vert \Sigma^{1/2}\cdot\Vert_p}}(\Sigma^{1/2}a).
\end{align*}
Let us assume without loss of generality that all the $a_i\geq 0$. and let us show now that  
\begin{align*}
     0 \leq \Sigma^{-1/2}\mbox{proj}_{B^d_{\Vert \Sigma^{1/2}\cdot\Vert_p}}(\Sigma^{1/2}a)\leq a
\end{align*}
First note that for all the coordinates of the vector $\Sigma^{-1/2}\mbox{proj}_{B^d_{\Vert \Sigma^{1/2}\cdot\Vert_p}}(\Sigma^{1/2}a)$ are nonnegative. Indeed, if some coordinate of the projection were negative, then by simply replacing it by 0, we could reduce the total cost of projection and therefore contradicts its optimality. Then using the fact that $\Sigma$ is diagonal, we want to show that for any $v\in\mathbb{R}^d$
\begin{align*}
     0 \leq \mbox{proj}_{B^d_{\Vert \Sigma^{1/2}\cdot\Vert_p}}(\Sigma^{1/2}a)\leq \Sigma^{1/2}a
\end{align*}
where the inequality is coordinate-wise.  Now remarks that we have
\begin{align*}
     \mbox{proj}_{B^d_{\Vert \Sigma^{1/2}\cdot\Vert_p}}(\Sigma^{1/2}a) = \Sigma^{-1/2} \mbox{proj}_{B^d_{\Vert\cdot\Vert_p}}^{\Sigma^{-1}}(\Sigma a)
\end{align*}
where 
\begin{align*}
\mbox{proj}_{B^d_{\Vert\cdot\Vert_p}}^{\Sigma^{-1}}(v):= \text{argmin}_{\Vert x\Vert_p\leq 1}\Vert x - v\Vert_{\Sigma^{-1}} 
\end{align*}
Therefore for any $p\geq 2$, the first order condition gives on $x:=\mbox{proj}_{B^d_{\Vert\cdot\Vert_p}}^{\Sigma^{-1}}(\Sigma a)$ that for all coordinate $i\in[d]$,
\begin{align*}
&(1/\lambda_i) (x_i - \lambda_i a_i) +\nu x_i^{p-1}=0,\,x_i\geq 0,
\end{align*}
for some $\nu\geq 0$. The above further gives
\begin{align*}
  (1/\lambda_i) x_i + \nu x_i^{p-1} = a_i,\,\forall i \in [d].
\end{align*}
It follows that  $\tilde{x}:=\mbox{proj}_{B^d_{\Vert \Sigma^{1/2}\cdot\Vert_p}}(\Sigma^{1/2}a)$ must satisfy:
\begin{align*}
  1/(\sqrt{\lambda_i}) \tilde{x}_i + \nu (\sqrt{\lambda_i}\tilde{x}_i)^{p-1} = a_i  
\end{align*}
which is equivalent to
\begin{align*}
 \tilde{x}_i + \sqrt{\lambda}_i\nu (\sqrt{\lambda_i}\tilde{x}_i)^{p-1} = \sqrt{\lambda_i}a_i  
\end{align*}
However if $\tilde{x}_i> \sqrt{\lambda_i}a_i $, then by nonnegativity of $ \sqrt{\lambda}_i\nu (\sqrt{\lambda_i}\tilde{x}_i)^{p-1}$ we obtain a contraction and the result follows.
\end{proof}

The following is a generalization to the previous lemma.
\begin{restatable}{lm}{}
\label{general-lm-proj}
Let $C$ be a closed convex subset of $\mathbb R^d$ which is symmetric about the coordinate axes. Let $a \in \mathbb R^d$ and set $b = \mbox{proj}_C(a)$. Then, it holds coordinate-wise that $\sign(b)=\sign(a)$ and $|b| \le |a|$.
\label{lm:kolmogorov}
\end{restatable}
For example, if $\Sigma$ is diagonal and the attacker's norm $\|\cdot\|$ is coordinate-wise even (meaning that $\|x\|=\|z\|$ whenever $z$ is obtained from $x$ by flipping the sign of some or no coordinate), then the ball $B_{\|\Sigma^{1/2}\cdot\|}^d$ is symmetric about the coordinate axes, and we recover the previous lemma.
\begin{proof}[Proof of Lemma \ref{lm:kolmogorov}]
Recall the Kolmogorov characterization of the equation $b=\mbox{proj}_C(a)$, namely
\begin{eqnarray}
(a-b)^\top (c - b) \le 0,\,\forall c \in C.
\label{eq:projoptcond}
\end{eqnarray}
Now fix an index $i \in [d]$. If $b_i = a_i$, there is nothing to show. Otherwise, suppose $a_i \ge 0$. Then $b_i \ge 0$, since $C$ is symmetric about the $i$th axis in $\mathbb R^d$ (see details further below). Also, $C$ must contain $b-te_i$, for sufficient small $t > 0$; otherwise $b_i=a_i$. Here, $e_i$ is the $i$th standard basis vector in $\mathbb R^d$. Applying \eqref{eq:projoptcond} with $c=b-te_i$ then gives $(a-b)^\top (-t e_i) \le 0$, i.e $-ta_i + tb_i \le 0$, i.e $0 \le b_i \le a_i$. Similarly, if $a_i \le 0$, consider $c  = b + te_i$ instead, to deduce $a_i \ge b_i \le 0$. We conclude that $|b_i| \le |a_i|$ as claimed.

To conclude the proof we need to provide one omitted detail, namely, that $b_i$ has the same sign as $a_i$. We proceed be \emph{reductio ad absurdum}. So, suppose $a_ib_i < 0$. Let $b'$ be obtained from $b$ by flipping the sign $i$th coordinate. Because $C$ is symmetric around the $i$th axis, it must contain $b'$. One then computes
$$
\|b'-a\|^2_2 - \|b-a\|^2_2 = (b_i-a_i)^2 - (b_i-a_i)^2 = (b_i+a_i)^2 - (b_i-a_i)^2 = 4a_i b_i < 0,
$$
which contradicts the optimality of $b$ as the point of $C$ which is closest to $a$.
\end{proof}
% \begin{thm}
% Let $\Sigma$ be a diagonal definite positive matrix and $p\geq 2$.
% Then we have 
% \begin{align*}
%     \inf_{B\in\mathcal{S}_d^{++}(\mathbb{R}),t\geq 0} \frac{E^{\Vert\cdot\Vert_p}(B^{1/2}w^{B}_c(t,w_0),w_0,r)}{E^{\Vert\cdot\Vert_p}_{opt}(r,w_0)}\leq \alpha
% \end{align*}
% \end{thm}
% \begin{proof}

We are now ready to prove Theorem \ref{thm-lp-optimal}.
\begin{proof}[Proof of Theorem \ref{thm-lp-optimal}]
Because $\Sigma:=\diag(\lambda_1,\dots,\lambda_d)$ is diagonal let $M=\diag(\sigma_1,\dots,\sigma_d)$ with $\sigma_i>0$ and let us consider $B=M\Sigma^{-1}$. In that case we obtain that:
\begin{align*}
     B^{1/2}w_c^{B^{1/2}}(t,w_0)&:= B^{1/2}(I_d - \exp(-tB^{1/2}\Sigma B^{1/2})^t))B^{-1/2}w_0\\
     &=(I_d - \exp(-tM))w_0
\end{align*}
Then using Lemma~\ref{lm:order-prox}, the minimizer, $w_{opt}$  of $w\mapsto\widetilde{E}^{\Vert\cdot\Vert}(w,w_0,r)$ satisfies in a coordinate-wise manner:
\begin{align*}
    |w_{opt}|\leq |w_0|,\,\mathrm{sign}(w_{opt}) = \mathrm{sign}(w_0).
\end{align*}
If $(w_{0})_i=0$ then by the above conditions we have necessarily $(w_{opt})_i=0$ and therefore $\sigma_i>0$ can be chosen arbitrarily.

Now, let us consider the case where $(w_{0})_i\neq 0$. Fix $t>0$. If $|(w_{0})_i|>|(w_{opt})_i|>0$, we can just take $\sigma_i$ such that:
\begin{align*}
    1-\exp(-t\sigma_i) = \frac{|(w_{opt})_i|}{|(w_{0})_i|},\text{ i.e } \sigma_i = \sigma_i(t)= -\frac{1}{t}\log\left(1 -  \frac{|(w_{opt})_i|}{|(w_{0})_i|}\right).
\end{align*}
On the other hand, if $|(w_{0})_i|=|(w_{opt})_i|$, we can simply consider a positive sequence of $(\sigma_{i,m})_{m\geq 0}$ such that
$\sigma_{i,m}\underset{m\rightarrow \infty}{\longrightarrow}\infty$. Finally, if $|(w_{0})_i|> |(w_{opt})_i|=0$, then we can consider a positive sequence of $(\sigma_{i,m})_{m\geq 0}$ such that
$\sigma_{i,m}\underset{m\rightarrow \infty}{\longrightarrow}0$. Finally by considering such sequence we obtain that 
\begin{align*}
    \lim_{m\rightarrow \infty}  B_m^{1/2}w^{B_m}(t,w_0) \rightarrow w_{opt}
\end{align*}
Then, by using the fact that
\begin{align*}
    \lim_{m\rightarrow \infty}E^{\Vert\cdot\Vert_p}(B_m^{1/2}w^{B_m}(t,w_0),w_0,r) &\leq \lim_{m\rightarrow +\infty}\widetilde{E}^{\Vert\cdot\Vert_p}(B_m^{1/2}w^{B_m}(t,w_0),w_0,r)\\
    &=\inf_w \widetilde{E}^{\Vert\cdot\Vert_p}(w,w_0),w_0,r)\\
    &\leq \alpha \inf_w E^{\Vert\cdot\Vert_p}(w,w_0),w_0,r),
\end{align*}
and the desired result follows.
% \end{proof}
\end{proof}

\section{Proof of Theorem~\ref{thm:stats}: Consistency of Proposed Two-Stage Estimator (Algorithm \ref{alg:twostage})}
\stats*

\begin{proof}
For $r \ge 0$, vectors $w,a \in \mathbb R^d$ and a psd matrix $C \in \mathbb R^{d\times d}$, define $F(w,a,C,r) := \|w-a\|_C + r\|w\|_\star$. Note that the adversarial risk proxy $\widetilde E$ defined in \eqref{eq:Etilde} can be written as $\widetilde E^{\Vert\cdot\Vert}(w,w_0,r) = F(w,w_0,\Sigma,r)^2$, for all $w \in \mathbb R^d$. Further more, by Lemma \ref{lm:sandwich}, we know that
\begin{eqnarray}
\label{eq:sandwichF}
E^{\Vert\cdot\Vert}(w,w_0,r) \le F(w,w_0,\Sigma, r)^2 \le \alpha \cdot E^{\Vert\cdot\Vert}(w,w_0,r).
\end{eqnarray}
Thus, the second line of inequalities in the theorem follows from the first.
Thus, we only prove the first line.
% The proof presented below will exploit this: we shall analyze the function $F$ instead of $E$.

So, let $w_{opt}$ be any minimizer of $F(w,w_0,\Sigma,r)$ over $w \in \mathbb R^d$. Observe that
\begin{eqnarray}
\begin{split}
F(\widehat w,w_0,\Sigma,r) - F(w_{opt},w,w_0,\Sigma, r) &= F(\widehat w,w_0,\Sigma,r) - F(\widehat w,\widehat w_0,\widehat \Sigma, r) + F(\widehat w,\widehat w_0, \widehat \Sigma,r) - F(w_{opt},\widehat w_0,\widehat \Sigma, r)\\
&\quad\quad + F(w_{opt}, \widehat w_0, \widehat \Sigma,r) - F(w_{opt}, w_0, \Sigma, r)\\
&\le F(\widehat w, w_0, \Sigma,r) - F(\widehat w,\widehat w_0,\widehat \Sigma, r) + F(w_{opt}, \widehat w_0, \widehat \Sigma,r) - F(w_{opt}, w_0, \Sigma, r),
\end{split}
\end{eqnarray}
where we have used the fact that $F(\widehat w,\widehat w_0, \widehat \Sigma,r) - F(w_{opt},\widehat w_0,\widehat \Sigma, r) \le 0$, since $\widehat w$ minimizes $w \mapsto F(w,\widehat w_0,\Sigma,r)$ by construction. We deduce that
\begin{eqnarray}
F(\widehat w,w_0,\Sigma,r) \le F(w_{opt},w,w_0,\Sigma, r) + e_4(\widehat w) + e_4(w_{opt}),
\label{eq:longtriangle}
\end{eqnarray}
where $e_4(w) := |F(w, w_0, \Sigma,r) - F(w,\widehat w_0,\widehat \Sigma, r)|$.
The rest of the proof is divided into two steps.

\textbf{Step 1: Controlling $e_4(w_{opt})$.}
Now, for any $w \in \mathbb R^d$, repeated application of the triangle-inequality gives
\begin{eqnarray}
\begin{split}
&e_4(w) := |F(w,w_0,\Sigma,r)-F(w,\widehat w_0,\widehat \Sigma,r)| = |\|w-w_0\|_{\Sigma}-\|w-\widehat w_0\|_{\widehat \Sigma}|\\
&= |\|w-w_0\|_{\Sigma}-\|w-\widehat w_0\|_{\Sigma} + \|w-\widehat w_0\|_{\Sigma} - \|w-\widehat w_0\|_{\widehat \Sigma}|\\
% &= |\|w-w_0\|_{\Sigma}-\|w-\widehat w_0\|_{\Sigma} + \|w-\widehat w_0\|_{\Sigma} - \|w-\widehat w_0\|_{\widehat \Sigma}| \\
&\le \|\Sigma\|_{op}\cdot e_1 + \|w-\widehat w_0\|_2\cdot e_2 \lesssim e_1 + \|w-\widehat w_0\|_2\cdot e_2.
\end{split}
\label{eq:e4}
\end{eqnarray}
Let $\gamma_r(w_0) := \inf_{w \in \mathbb R^d} F(w,w_0,\Sigma,r)$. We deduce that
\begin{eqnarray}
\begin{split}
e_4(w_{opt}) &\le e_1 + \|w_{opt}-\widehat w_0\|_2\cdot e_2 \le e_1 + (\|w_{opt}-w_0\|_2+\|\widehat w_0-w_0\|_2)\cdot e_2\\
&\le e_1 + (\gamma_r(w_0)+e_1)\cdot e_2 \lesssim e_1 + \gamma_r(w_0)\cdot e_2 \lesssim e_1 + \|w_0\|_{\Sigma}\cdot e_2,
\end{split}
\end{eqnarray}
where the last step is thanks to the fact that $\gamma_r(w_0) \le \|w_0\|_\Sigma$.

\textbf{Step 2: Controlling $e_4(\widehat w)$.}
If we take $w=\widehat w$ to be a minimizer of the function $F(w,\widehat w_0,\widehat \Sigma,r)$ over $\mathbb R^d$, then
\begin{eqnarray*}
\begin{split}
\lambda_{\min}(\widehat \Sigma)\| w-\widehat w_0\|_2^2 &\le \|w-\widehat w_0\|_{\widehat \Sigma}^2 = \|w-\widehat w_0\|_{\Sigma}^2 + (w-\widehat w_0)^\top (\widehat \Sigma-\Sigma)(w-\widehat w_0) \\
&\le \|w-\widehat w_0\|_{\Sigma}^2 + \|w-\widehat w_0\|_{\Sigma}^2 \cdot e_2\\
&\le \gamma_r(\widehat w_0)^2 + \|w-\widehat w_0\|_{\Sigma}^2 \cdot e_2\\
&\le \|\widehat w_0\|_{\Sigma}^2 + \|w-\widehat w_0\|_{\Sigma}^2\cdot  e_2.
\end{split}
\end{eqnarray*}
We deduce that
\begin{eqnarray}
\begin{split}
(1-e_2)\|w-\widehat w_0\|_2^2 &\le \|\widehat w_0\|_{\Sigma}^2 \le (\|w_0\|_{\Sigma}+e_1)^2,
\end{split}
\end{eqnarray}
and so $\|w-\widehat w_0\|_2 \lesssim \|w_0\|_{\Sigma}+e_1$. Plugging this into \eqref{eq:e4} gives
\begin{eqnarray}
e_4(\widehat w) \lesssim e_1 + \|w_0\|_{\Sigma} \cdot e_2.
\end{eqnarray}
Putting things together establishes that
\begin{eqnarray}
F(\widehat w, w_0,\Sigma,r) \le F(w_{opt}, w_0,\Sigma,r) + O(e_1) + O(\|w_0\|_{\Sigma} \cdot e_2),
\label{eq:almostthere}
\end{eqnarray}
from which the result follows.
\end{proof}

\end{document}